\RequirePackage{algorithm}

\documentclass[smallextended]{svjour3}

\usepackage[T1]{fontenc}
\usepackage{lmodern}
\usepackage[utf8]{inputenc}

\usepackage{natbib}
\usepackage{float}
\usepackage{amsmath}
\usepackage{amssymb}
\usepackage{caption}
\usepackage{enumerate}
\usepackage{graphicx}
\usepackage{hyperref}
\usepackage{url}            
\usepackage{booktabs}       
\usepackage{amsfonts}       
\usepackage{nicefrac}       
\usepackage{microtype}      
\usepackage{wrapfig}
\usepackage{subfigure}

\usepackage{color}

\usepackage{makecell}
\usepackage{multirow}
\usepackage{hyperref}

\def\R{\mathbb{R}}

\def\cov{\mathrm{cov}}
\def\diag{\mathrm{diag}}
\def\nor{\mathcal{N}}
\def\Cov{\textnormal{Cov}}

\def\1{\mathds{1}}

\def\X{\mathbf{X}}
\def\Y{\mathbf{Y}}

\def\SD{\mathcal{SD}}
\def\DE{\mathcal{DE}}

\def\mWeICA{\text{MWeICA}}

\newtheorem{observation}{Observation}[section]

\begin{document}

\title{Independent Component Analysis based on multiple data-weighting}

\author{Andrzej Bedychaj         \and
        Przemysław Spurek \and
        Łukasz Struski  \and
        Jacek Tabor
}

\institute{Andrzej Bedychaj \email{andrzej.bedychaj@gmail.com}           \and
           Przemysław Spurek \email{przemyslaw.spurek@gmail.com} \and
           Łukasz Struski \email{lukaszstruski@gmail.com}\and
           Jacek Tabor \email{jcktbr@gmail.com}
}

\date{Received: date / Accepted: date}

\maketitle

\begin{abstract}

Independent Component Analysis (ICA) - one of the basic tools in data analysis - aims to find a coordinate system in which the components of the data are independent. 
In this paper we present \textbf{M}ultiple-\textbf{we}ighted \textbf{I}ndependet \textbf{C}omponent \textbf{A}nalysis (\mWeICA{}) algorithm, a new ICA method which is based on approximate diagonalization of weighted covariance matrices. Our idea is based on theoretical result, which says that linear independence of weighted data (for gaussian weights) guarantees independence.
Experiments show that \mWeICA{} achieves better results to most state-of-the-art ICA methods, with similar computational time. \end{abstract}

\section{Introduction}

%
%

Independent Component Analysis (ICA), called also Blind Source Separation (BSS),  is a method for decomposing mixture of signals into a set of independent components. ICA is similar in many aspects to principal component analysis (PCA). In PCA we look for an orthonormal change of basis so that the components are not linearly dependent (uncorrelated). ICA can be described as a search for the optimal basis (coordinate system) in which the components are independent.
Although both problems are closely related, PCA has a closed-form solution given by simple matrix operations, while most of existing solutions of ICA use iterative optimization procedure. 

In signal processing ICA is a computational method for separating a multivariate signal into additive subcomponents and has been applied in magnetic resonance \citep{beckmann2004probabilistic}, MRI \citep{beckmann2005tensorial,rodriguez2012noising}, EEG analysis \citep{brunner2007spatial,delorme2007enhanced},
fault detection \citep{choi2005fault}, financial time series \citep{kiviluoto1998independent} and  seismic recordings \citep{haghighi2008ica}.
Moreover, it is hard to overestimate the role of ICA in pattern recognition and image analysis; its applications include face recognition \citep{yang2005kernel,dagher2006face}, texture segmentation \citep{jenssen2003independent}, object recognition~\citep{bressan2003using}, multi-label learning \citep{xu2016local} and feature extraction \citep{lai2014multilinear}.

\begin{figure}[!h]
\begin{center}
\subfigure[Original images 42049 and 220075.] {\label{fig:image_ICA_int_1}
\includegraphics[width=.23\linewidth]{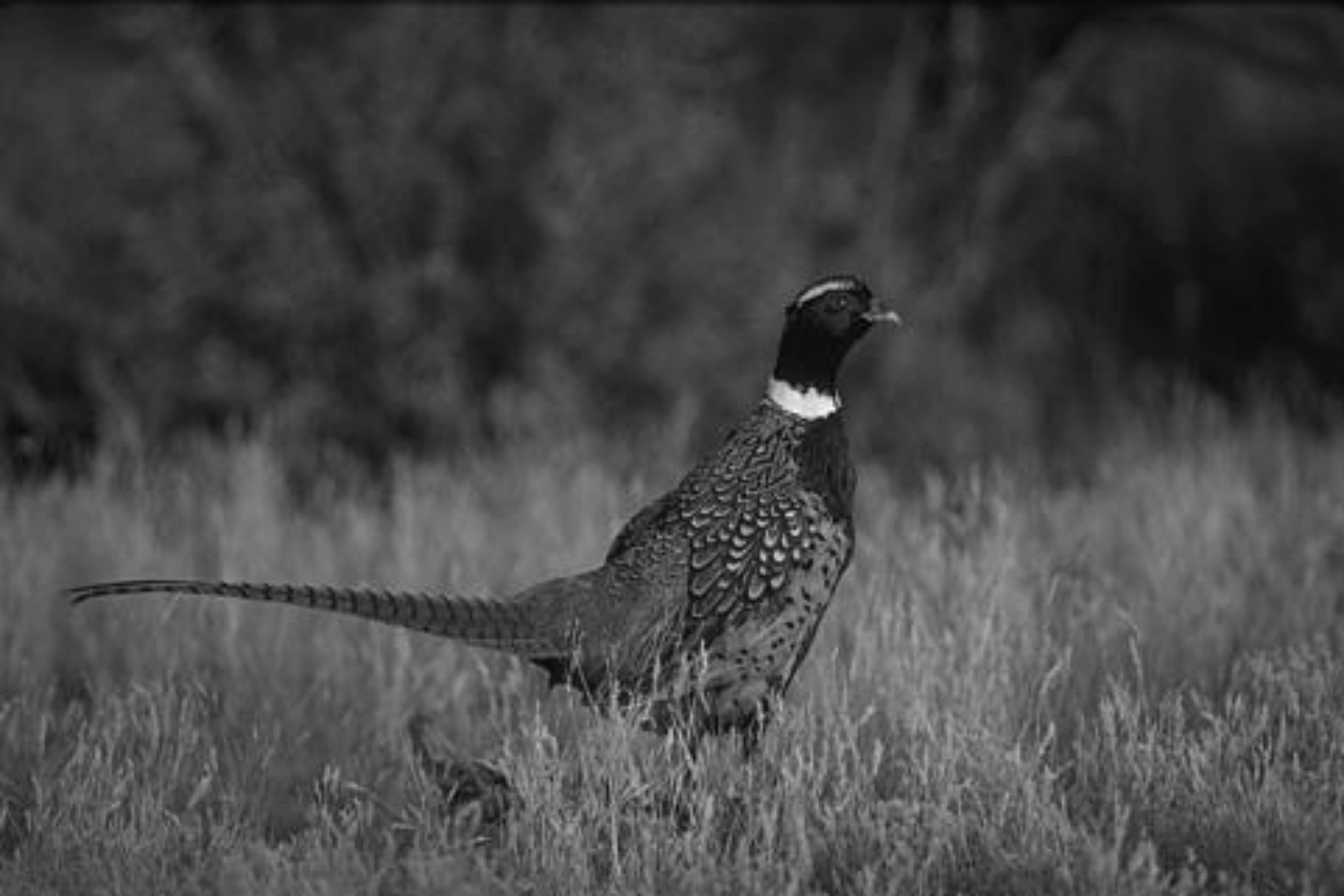}
\includegraphics[width=.23\linewidth]{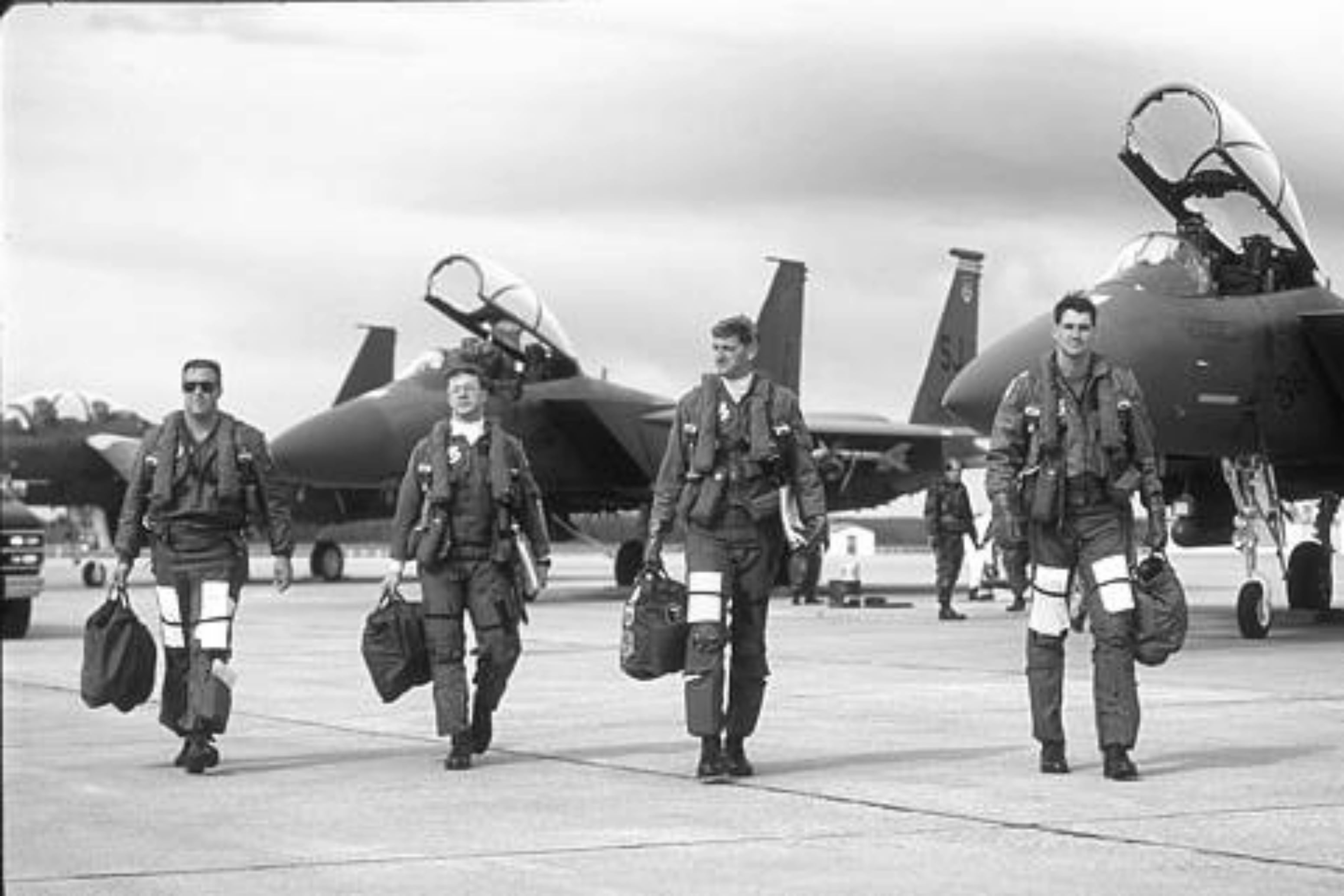}
} 
\subfigure[Mix of images done via random linear projection.] {\label{fig:image_ICA_int_2}
\includegraphics[width=.23\linewidth]{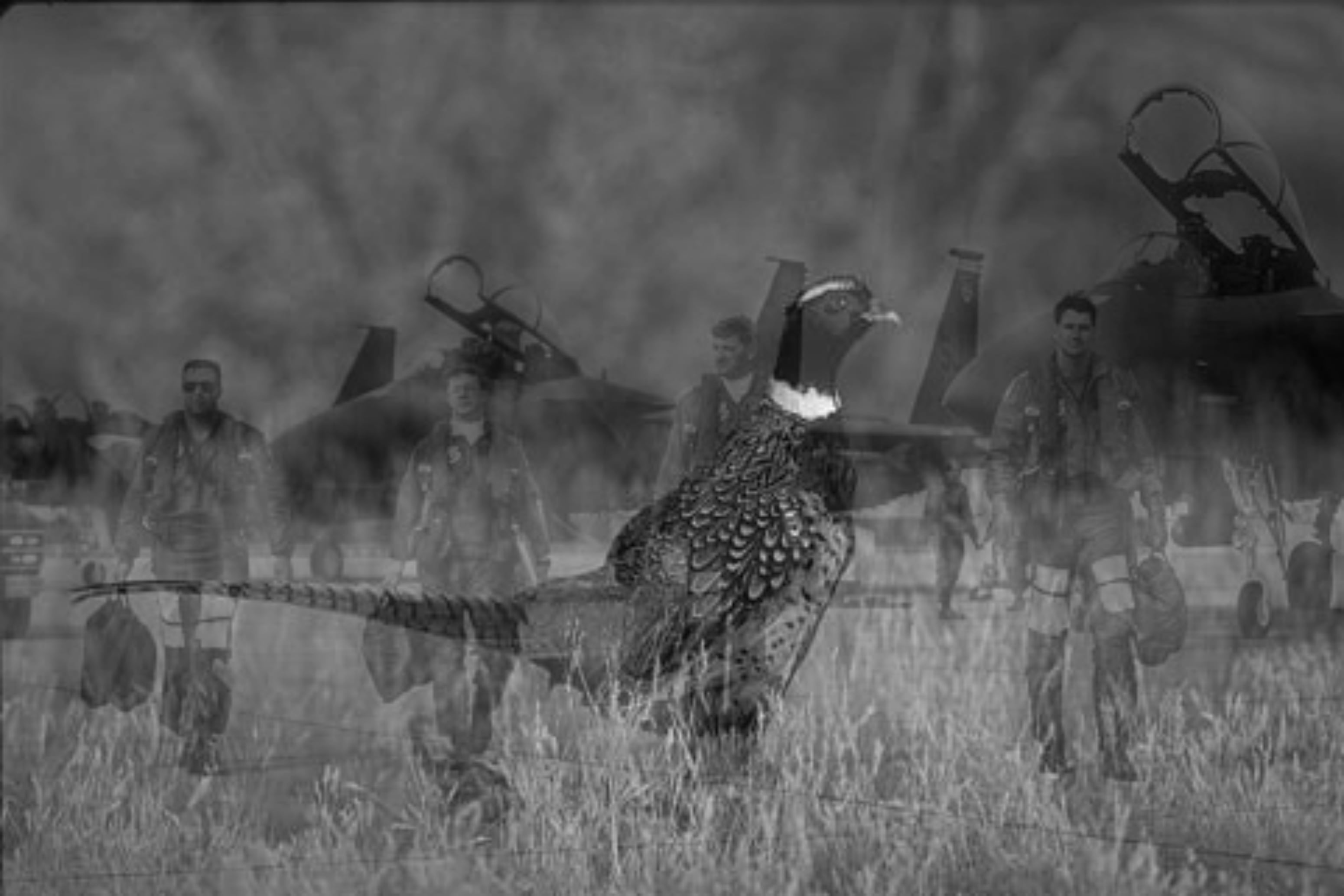} 
\includegraphics[width=.23\linewidth]{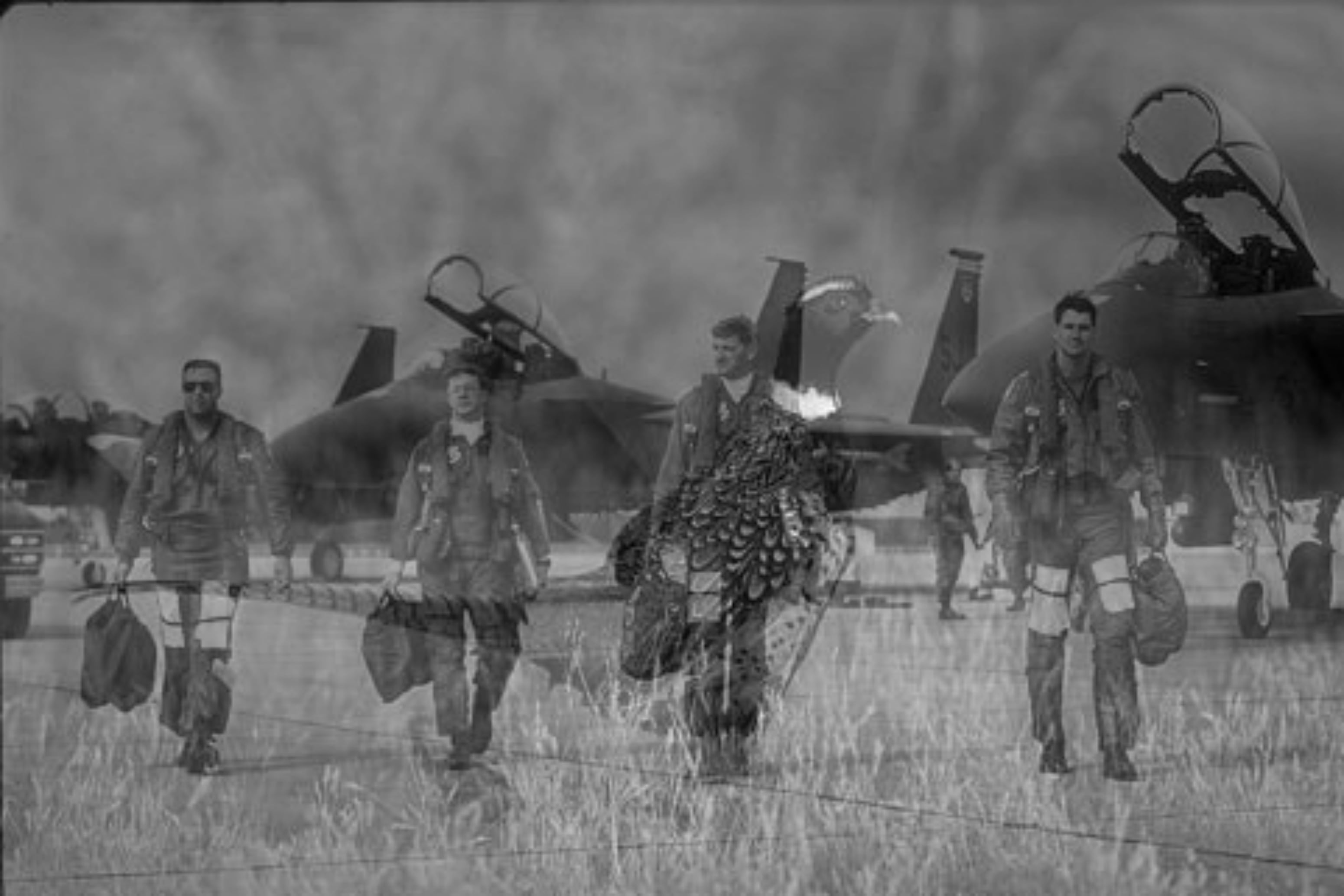}
} 
\subfigure[\mWeICA.] {\label{fig:image_ICA_int_3}
\includegraphics[width=.23\linewidth]{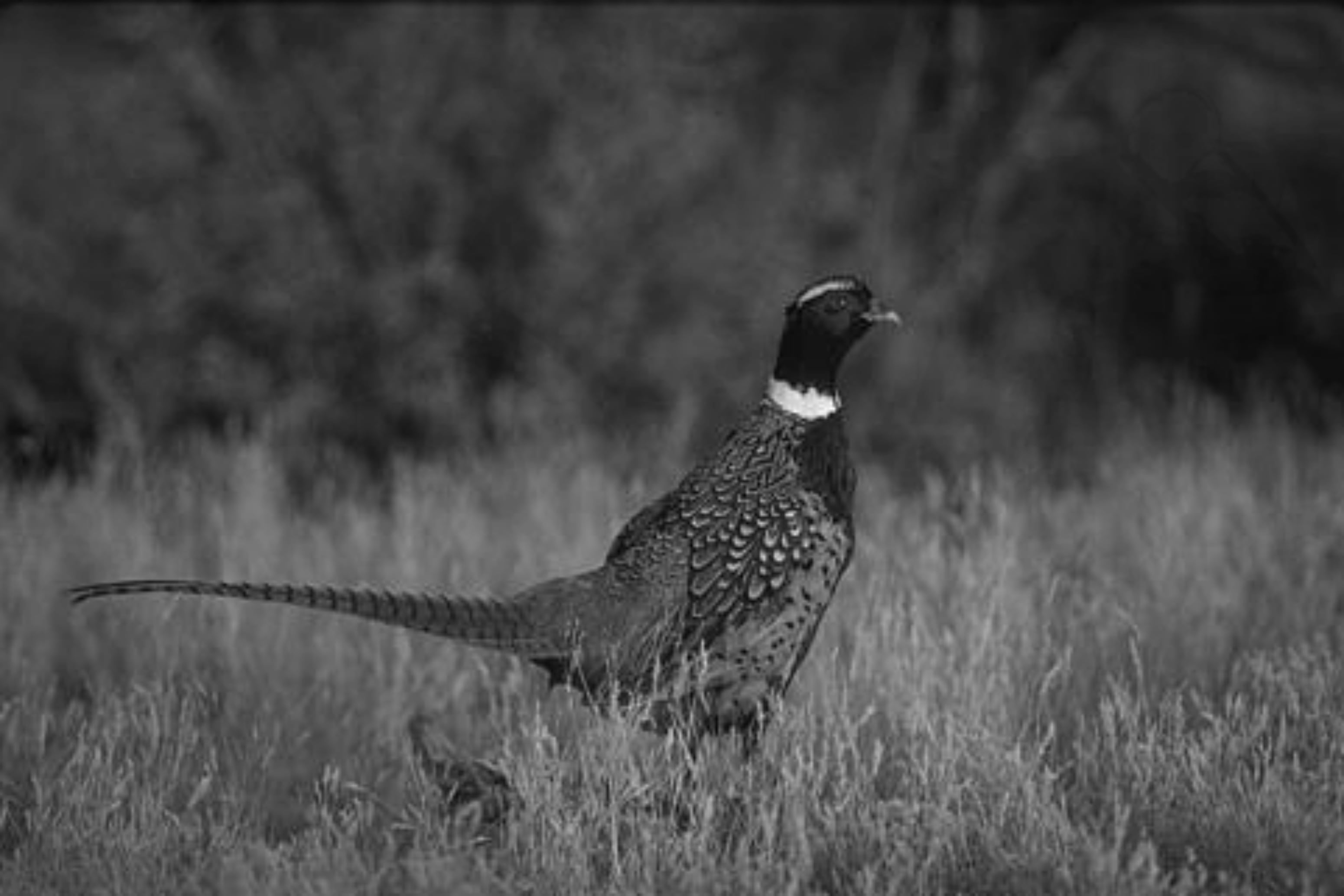} 
\includegraphics[width=.23\linewidth]{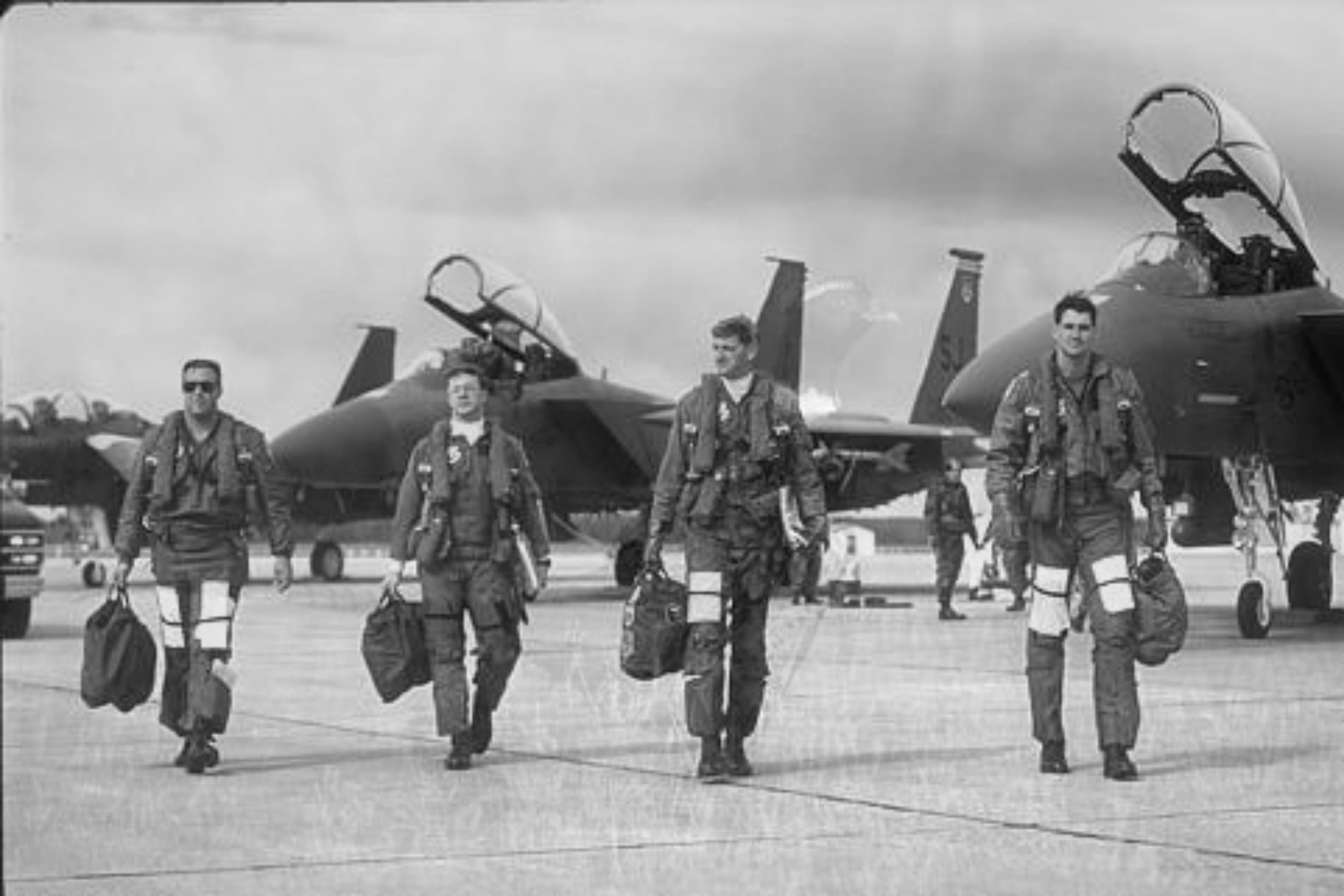} 
} 
\subfigure[FastICA.] {\label{fig:image_ICA_int_4}
\includegraphics[width=.23\linewidth]{picture1/130-3mWeICA1} 
\includegraphics[width=.23\linewidth]{picture1/130-3mWeICA0} 
}
\end{center}
\caption{Comparison of images separation by our method (\mWeICA) with FastICA.}
\label{fig:image_ICA_int}
\end{figure}

Let us now briefly describe the most common approaches used in solving ICA problem.
Lacoume and Ruiz in \citep{lacoume1992separation} where one of the first to  use higher-order statistics in case of blind source separation. Algorithm that separates observed mixed signals into latent source signals by exploiting fourth order moment was introduced in \citep{cardoso1999high}. Cardoso applied fourth-order cumulants (aforementioned kurtosis), as a measure for fitting independent components (this method is called JADE). The main drawback of those approaches is that kurtosis is very sensitive to the outliers, which makes some difficulty in its estimation from small samples \citep{yang2005ica}. Applying lower-order moments for ICA is not exploited that much in literature. Independent component analysis using score functions from the Pearson system is one of the most renowned method exploring that subject (PearsonICA \citep{karvanen2002blind,Juh-2002}). The algorithm is designed especially for problems with asymmetric sources. 
Split Gaussian ICA (SgICA) \citep{spurek2017ica} is based on the maximum likelihood estimation. In such a case we search for the coordinate system optimally fitted to data
as well as the marginal densities such that the data density factors in the base are the product of marginal densities. Authors model skewness using the Split Gaussian distribution, which is well adapted to asymmetric data.

Another important approach to identifying independent components is related to mutual information measure, that is also a measure of independence of base signals \citep{bell1995information,ICA:1}. One of the fastest realization of such approach is FastICA \citep{Hyvarinen-1999}. Algorithm revolves around extracting prewhiten components one by one, using nonlinear function (proposed in \citep{Hyvarinen2001}) in fixed-point iterative approach. ProDenICA \citep{bach2002kernel,hastie2009elements} expands single nonlinear function to the entire function space of candidate nonlinearities making it more robust to varying source distributions, but also more time consuming.

An approach based on (approximate) diagonalization of matrices 
to ICA (which we also apply in different context) was proposed in \citep{Eidinger-2004}. Authors created an algorithm named CHESS (CHaracteristic function Enabled Source Separation). Solution proposed in aforementioned paper achieves separation by applying joint diagonalization to a set of estimated second derivative matrices (Hessians) of the second generalized characteristic function at selected processing points of mixed dataset. In \citep{spurek2018fast} authors present ICA method called WeICA (Weighted ICA), which is also based on simulatenous diagonalization of two matrices, and consequently has a simple closed-form solution.  
WeICA uses weighted data to determine independent components. The approach proposed in \citep{spurek2018fast} outperforms other state-of-the-art ICA methods with respect to time complexity, gives very good results in the case of dimension reduction and can be used as a initialization for iterative approaches to ICA problem.
Unfortunately the method is unstable and gives slightly worse results in the case of source separation problem.


In this paper we want to propose a similar approach to WeICA, called Multiple Weighted ICA (\mWeICA), where the discriminating role is played by weighting of the data. \mWeICA{} is an easy to parallel algorithm for ICA task that takes advantage of approximate parallel diagonalization of weighted covariance matrices for base set $X$. As compared to WeICA, \mWeICA, while slower, gives better results in the case of source separation problem, see Fig. \ref{fig:image_ICA_int}. Moreover, in our main theoretical result, Theorem \ref{tw:odwrotne}, we show that 
the linear independence of normally weighted data guarantees independence. Consequently this allow us to construct a new measure of independence, which can be used similarly to dCov or dCor \citep{szekely2007measuring}.
Details will be covered in Section \ref{Theory} and full algorithm will be presented in Section \ref{sec:mwICA}.

\section{Weighted data}  \label{Theory}
%
%
%
%
%
%



Let $\X$ be a $d$-dimensional random vector with a probability density function $f$ and let $w:\R^d \to \R_+$ be a bounded weighting function. By $\X_w$ we denote a weighted random vector with a density
$$
f_w(x)= \frac{w(x)f(x)}{\int w(z)f(z)dz},
$$
which is just the normalization of $w(x) f(x)$. 

We recall that the random vector $\X$ with density $f$ in $\R^d$ has independent components iff $f$ factors as
$$
f(x_1,\ldots,x_d)=f_1(x_1) \cdot \ldots \cdot f_d(x_d),
$$
for a certain one dimensional $f_i$. Cleary, independence implies linear independence.
In general, except for multivariate gaussians, the opposite implication does not hold.

Let us begin with the observation that weighting by the normal density 
with covariance proportional to that of the random vector does not destroy the independence. By $\nor(m,\Sigma)$ we denote the normal density with mean at $m$ and covariance matrix $\Sigma$. Given a random vector $\X$, $m \in \R^d$ we put
$$ 
\X_{[m]}=\X_w \text{ with weight }w=\nor(m,\cov \X).
$$  

One can easily verify that for every affine map $Ax+b$, where $A$
is linear and $b \in \R^d$, we have
\begin{equation} \label{eq:lin}
A\X_{[m]}+b=(A\X+b)_{[Am+b]},
\end{equation}
\begin{equation} \label{eq:second}
\cov (A\X_{[m]}+b)=A \cov \X_{[m]} A^T.
\end{equation}

The above formula guarantee in particular that the ICA we are going to construct is invariant with respect to the affine transformations of the data.
As an important consequence of the fact that multivariate normal density factors as a product of univariate normal densities, we obtain the following observation.

\begin{observation} \label{pr:1}
Let $\X$ be a random vector in $\R^d$ with density $f$ which has indepenent components.
Let $m \in \R^d$ be arbitrary fixed. 
Then $\X_{[m]}$ has independent components.
\end{observation}
\begin{proof}
By the assumptions
\begin{equation} \label{eq:1}
f(x)=f(x_1,\ldots,x_d)=f_1(x_1) \cdot \ldots \cdot f_d(x_d),
\end{equation}
for certain densities $f_1,\ldots,f_d$. Since $f$ has indepenedent components, it has linearly independent components, which means that the covariance $\cov \X$ is diagonal, and therefore
$$
\nor(m,\cov\X)(x)=\nor_1(x_1) \cdot \ldots \cdot \nor_d(x_d),
$$
for certain one-dimensional gaussians $\nor_1,\ldots,\nor_d$. 
Consequently, by the above decomposition, $\X_{[m]}$ comes from a density which is the normalization of the function
$$
\begin{array}{l}
x \to \nor(m,\cov\X)(x)\cdot  f(x)
=\nor_1(x_1) f_1(x_1) \cdot \ldots \cdot \nor_d(x_d) f_d(x_d),
\end{array}
$$
which trivially means that the density of $\X_{[m]}$ has independent components.
\end{proof}

\section{Construction of \mWeICA{}}\label{sec:mwICA}

Let us first state formally the ICA problem. Given a random variable $\X$ we aim to find (if possible) an {\em unmixing matrix}, i.e. an invertible matrix $W$ such that $W^T \X$ has independent components. Now, directly from \eqref{eq:lin} and Observation~\ref{pr:1} we obtain the following proposition.

\begin{proposition} \label{pr:bas}
Let $\X$ be a random vector in $\R^d$ with density $f$ and let $W$
be an unmixing matrix for $\X$. Let $m \in \R^d$ be arbitrary fixed. Then $W$ is an unmixing matrix for $\X_{[m]}$, and consequently
$$
\cov (W^T\X_{[m]})=W^T \cov \X_{[m]} W \text{ is diagonal.}
$$
\end{proposition}

\begin{figure}
\centering
\includegraphics[height=6cm]{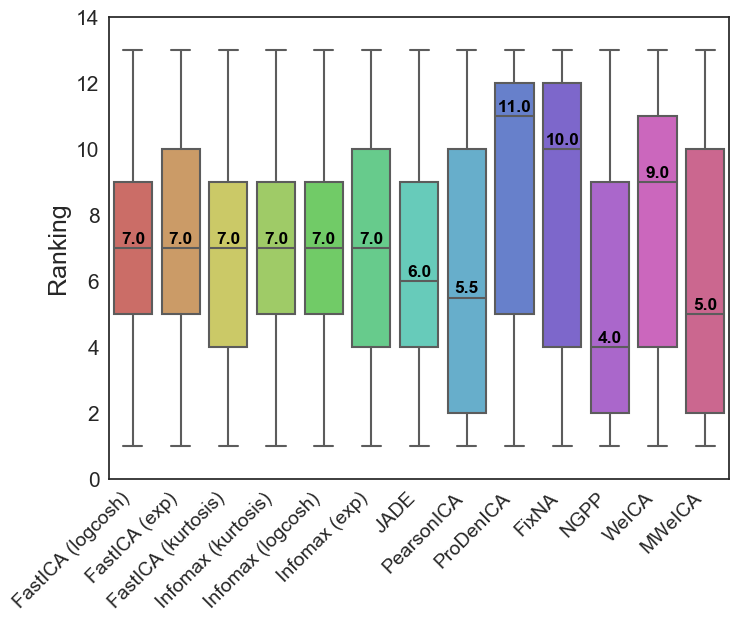}	
\caption{Comparison with popular implementations of ICA solutions (on set containing 1000 samples of $481\times 321$ pixel pictures). Plot presents boxplots of ranking (lower is better) places according to Tucker Congruency Coefficient \citep{Tucker} results.}
	\label{fig:imgaccuracy}
\end{figure}

The above proposition is the focal point of our idea. Before proceeding further, let us first recall
some basic results concerning simultaneous diagonalization of two matrices  \citep{fukunaga1990introduction,horn1985matrix}. We say that $W$ {\em diagonalizes} matrix $\Sigma$, if
$
W^T \Sigma W \text{ is diagonal.}
$
By $\SD(\Sigma_1,\ldots,\Sigma_n)$ we denote the set of matrices which simultaneously diagonalize all of the matrices:
$$
\SD(\Sigma_1,\ldots,\Sigma_n)=\{W:W \text{ diagonalizes }\Sigma_i
\text{ for all }i\}.
$$
It is well-known that for two positive symmetric matrices
the above set is nonempty, which is summarized in the following theorem, see \citep[Section 2.3]{fukunaga1990introduction}:



\begin{theorem}\label{tw:nonEmptySetOfDiagMatrix}
Let $\Sigma_1,\Sigma_2$ be symmetric positive matrices.
Then $\SD(\Sigma_1,\Sigma_2)$ is nonempty, and any its element $W$
 is given by eigenvector matrix of $\Sigma_1^{-1}\Sigma_2$.
 
Moreover, $W$ is determined uniquely (with respect to possible rescaling) if
$\Sigma_1^{-1}\Sigma_2$ has no multiple eigenvalues.
\end{theorem}

Applying the above theorem to Proposition \ref{pr:bas}, we directly obtain the following Corollary (a similar reasoning was applied in \citep{spurek2018fast} to construct WeICA):

\begin{corollary} \label{tw:weica}
Let $\X$ be a random vector and let $W$ be a matrix such that $W^T\X$
has independent components. Let $m_1,m_2 \in \R^d$ be given. Then 
\begin{equation} \label{eq:det}
W \in \SD(\cov \X_{[m_1]}, \cov \X_{[m_2]}).
\end{equation}

Moreover, if
\begin{equation} \label{eq:gen}
\cov^{-1} \X_{[m_1]} \cdot \cov \X_{[m_2]}
\text{ has distinct eigenvalues,}
\end{equation}
then $W$ is determined uniquely (up to possible rescaling), and consequently an arbitrary element of $\SD(\cov \X_{[m_1]}, \cov \X_{[m_2]})$ is an unmixing matrix for $\X$.
\end{corollary}

\begin{proof}
By the previous observations we conclude that $W$ simultaneously diagonalizes matrices $\Cov \X$ and $\Cov \X_{[m]}$. From the thesis of Theorem \ref{tw:nonEmptySetOfDiagMatrix} we conclude the proof.
\end{proof}

One can observe that Theorem \ref{tw:weica} can be used to determine the unmixing matrix for ICA problem, however, there appears the question of the choice of $m$. 
Morever, we can only obtain the estimators of the covariance from the sample, and consequently to obtain a more stable version we propose to take a randomly picked sample $m_1,\ldots,m_n$ from $X$:

\begin{observation}\label{tw:mweica}
Let $\X$ be random vector which has the unmixing matrix $W$. Let $m_1,\ldots,m_n$ be randomly drawn points. 
Then 
\begin{equation} \label{eq:s}
W \in \mathcal{SD}\left(\Cov \X_{[m_1]}, \ldots, \Cov \X_{[m_n]}\right).
\end{equation}
\end{observation}

We want to apply the above theorem in the case when we have only a sample from $\X$. Consequently if we place $X$ in place of $\X$ in \eqref{eq:s}, the weighted covariances will not be simultaneously diagonalizable. Thus to practically apply \eqref{eq:s} we need to use methods of approximate diagonalization, see \citep{Cardoso-1996, Pham-2001, Tichavsky-2009}. 
In our algorithm we apply \citep{Pham-2001} which  allows to calculate (approximately) unmixing matrix $W$ which minimizes the mean diagonalization error 
$$
\begin{array}{c}
\tfrac{1}{n}\sum_i \DE(W^T\Sigma_i W),
\end{array}
$$
where  the {\em diagonalization error $\DE(A)$} of a positive matrix $A$ is given by
$$
\begin{array}{c}
\DE(A)=\log \frac{\det \diag(A)}{\det A}.
\end{array}
$$
Clearly, $\DE(A) \geq 0$ and $\DE(A)=0$ iff $A$ is diagonal.

Thus the final \mWeICA{} can be stated as follows.

\paragraph{\mWeICA{} algorithm}
We are given a sample $X$ and a parameter $n$. We choose randomly $n$ elements $m_1,\ldots,m_n$. As an unmixing matrix for $X$ we take such an invertible matrix $W$ which minimizes\footnote{We find it with use of \citep{Pham-2001}.} the mean diagonalization error:
\begin{equation} \label{eq:q}
\DE(X;W;(m_i))=\frac{1}{n}\sum_i \DE(W^T \cov X_{[m_i]} W).
\end{equation}

Summarizing the reasoning from this section we see that 
\begin{itemize}
    \item  if $\X$ is a random vector such that $W^T\X$ has independent components for some invertible matrix $W$, then the value of RHS of \eqref{eq:q} is asymptotically\footnote{For the sample size of $X$ and $n$ going to infinity.} zero.
\end{itemize}
In the following section we prove our main theoretical result which show that also the opposite implication holds, i.e.
\begin{itemize}
    \item  if $\X$ is a random vector such that RHS of \eqref{eq:q} is asymptotically zero, then $W^T\X$ has independent components.
\end{itemize}

The above process can be expressed in following algorithm:
\begin{algorithm}[H]
\caption{\mWeICA}
\label{nWeICA}
Let $X=(x_j)_{j=1..k}$ be a dataset and let $n$ be given. To retrieve the unmixing matrix we proceed with the following steps:
\begin{enumerate}
\item compute $\Sigma=\Cov X$,
\item randomly pick $n$ points $m_i$ from $X$,
\item for each $i=1..n$ calculate weighted mean and covariance:
$$
\overline m_i=\tfrac{1}{\sum \limits_{j=1}^k w_{ij}} \sum_{j=1}^k w_{ij} x_j,
$$
$$
\Cov X_{[m_i]}=\tfrac{1}{\sum \limits_{j=1}^k w_{ij}} \sum_{j=1}^k w_{ij} (x_j-\overline m_i)(x_j-\overline m_i)^T,
$$
where $w_{ij}=\left(N(m_i,\Sigma)(x_j)\right)$ for $i=1..n,j=1,\ldots, k$,
\item retrieve by applying algorithm from \citep{Pham-2001} the best diagonalizing matrix $W$ for the set of $\{\Cov X_{[m_1]}, \dots, \Cov X_{[m_n]}\}$. 
   \end{enumerate}
Matrix $W$ is our unmixing matrix.
\end{algorithm}

Algorithm presented above is easy to parallelize. All operations from third point are independent from each other, which provides easy framework for concurrency. Diagonalization of covariance matrices is done via algorithm from \citep{Pham-2001} which was already implemented in Python \href{https://pythonhosted.org/pyriemann/index.html}{pyRiemann library}.

\begin{figure}[h]
	\centering
	\includegraphics[width=0.8\linewidth]{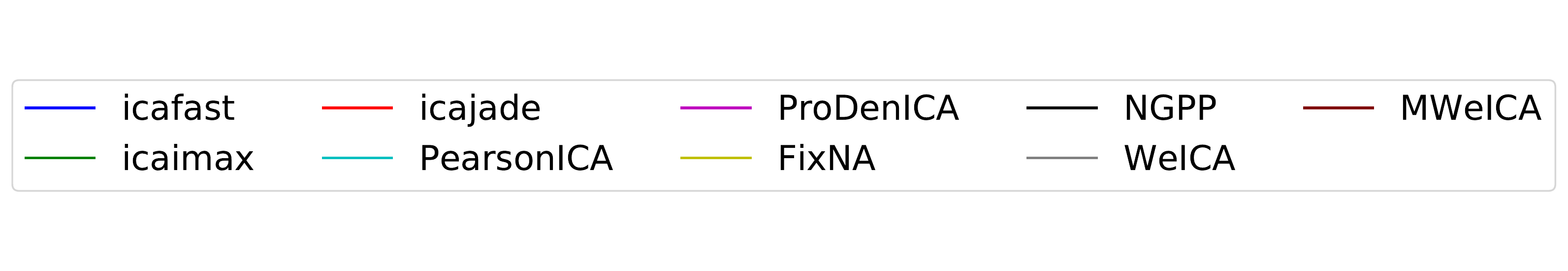} \\[-2em]
	\includegraphics[width=0.40\linewidth]{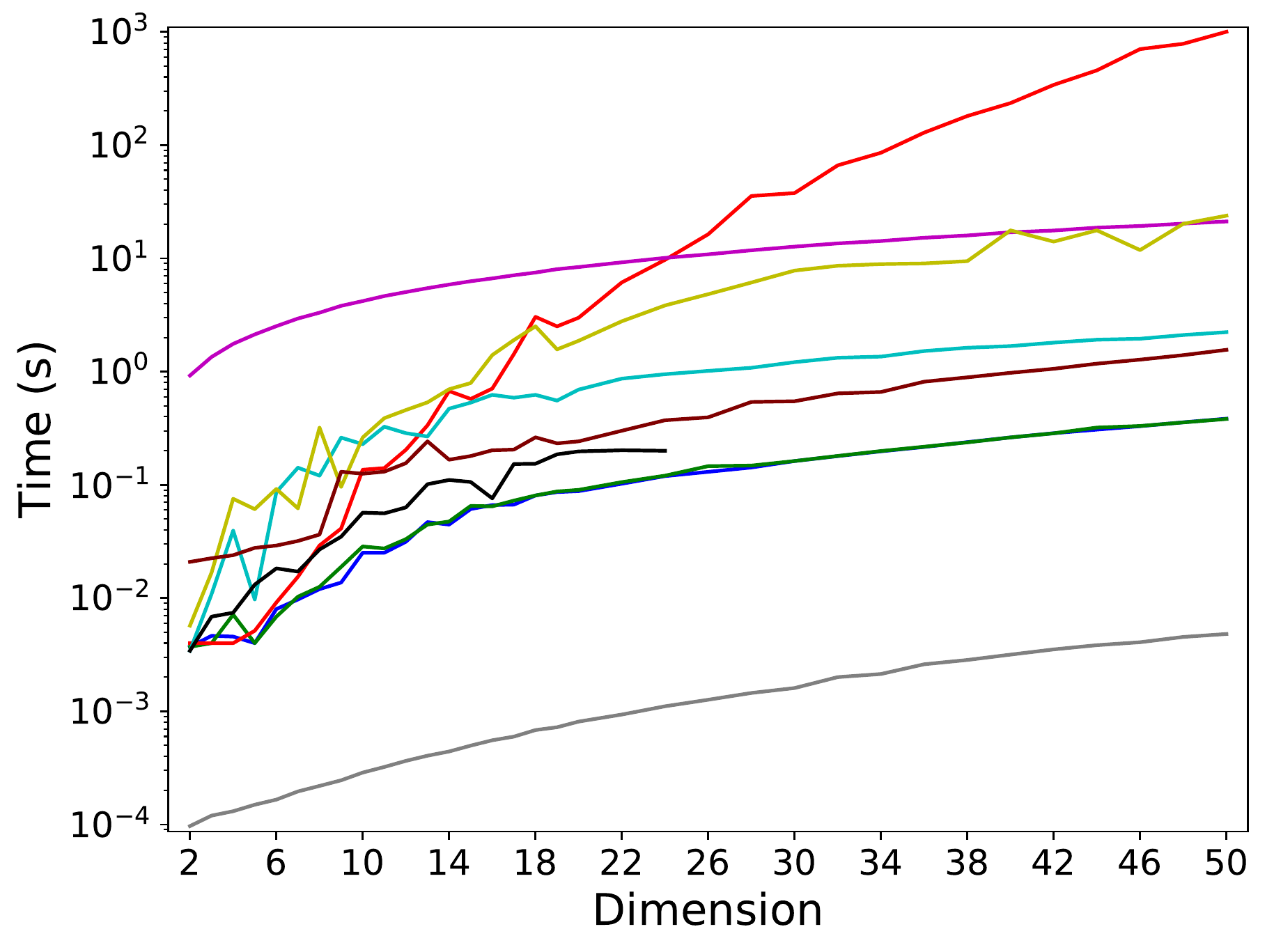}
	\includegraphics[width=0.40\linewidth]{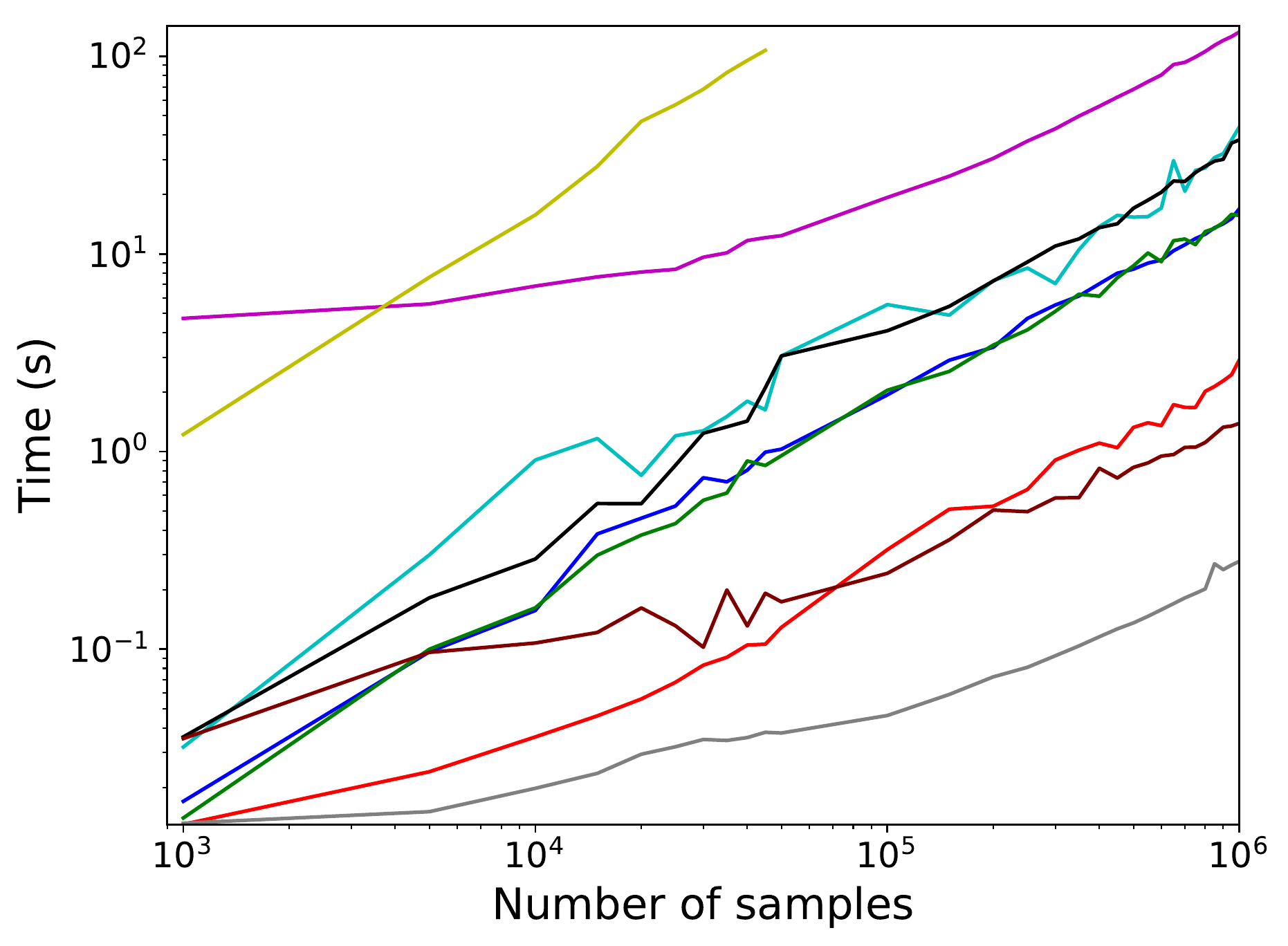}
	\caption{Comparison in time domain depending on mixed signal dimension (left hand side image) and on the number of data points (right hand side image)
	for ICA methods - only the fastest was tested. We can see that proposed method is highly insensitive to dimension and size of sample.} 
	\label{timeComparision}
\end{figure}

\section{Theory: independence index}

The foregoing observation gives an intuition that mathematical operations applied to unweighted data will not impact independence of further weighted data if the base set did not indicate any sign of that also. This allows us to work on unweighted data, and draw conclusions for later processed data based on those operations.

\begin{theorem} \label{th:lin}
We consider random vector $\Y$.
We assume that $\Y_{[m]}$ has linearly independent components for every $m \in B(\bar p,r) \subset \R^d$, for certain $r>0$. 

Then $\Y$ has independent components.
\end{theorem} 

\begin{proof}
For clarity of the proof we consider only the case $D=2$ (one can easily adapt it to fit the general case). 

By $f$ we denote the density of random vector $\Y$. We use the following notation
$$
M_{ij}(\phi_1,\phi_2)=\iint  v_1^iv_2^j \phi_1(v_1) \phi_2(v_2) f(v_1,v_2)dv_1dv_2,
$$
which corresponds to the weighted moments of order $i,j$ and $f$ is a density of our independent data $\Y$.

\begin{figure}[h]
\centering
	\includegraphics[width=0.49\linewidth]{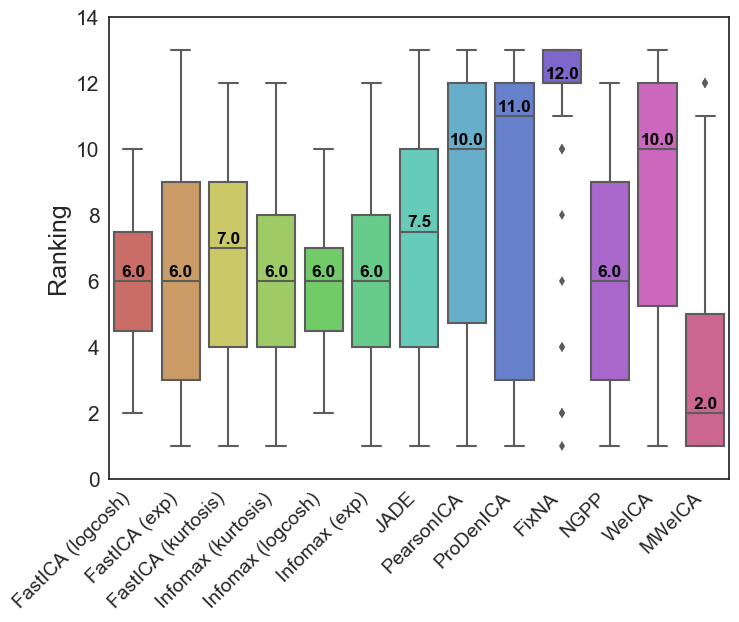}
	\includegraphics[width=0.49\linewidth]{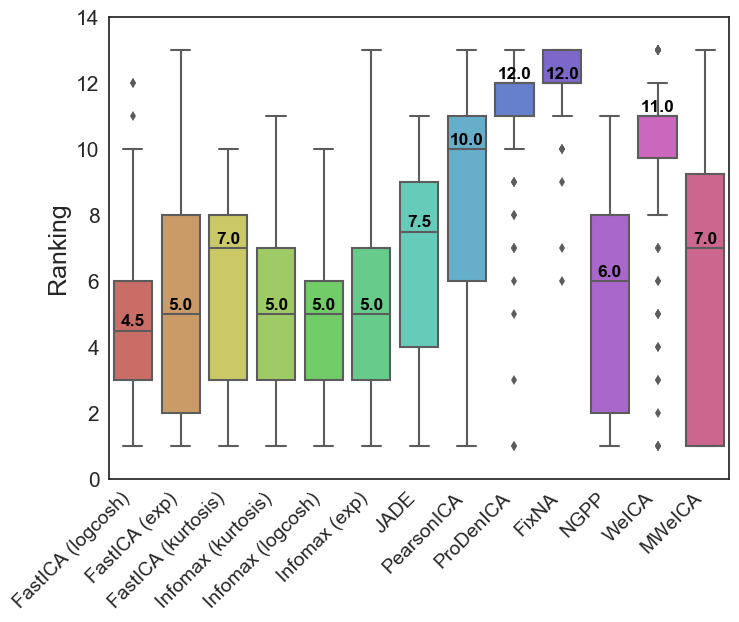}
	\caption{Results of ranking (lower is better) for Tucker Measure on synthetic bootstrap data.  On the left one can observe results for two dimensional problem solution, and on the right we can see results for three dimensional composition. }
	\label{rank_img_2}
\end{figure}
	
\textbf{STEP 1.}
Directly from the definition, the	linear independence of the weighted data means
$$
M_{00}(\phi_1,\phi_2)M_{11}(\phi_1,\phi_2)=M_{10}(\phi_1,\phi_2) M_{01}(\phi_1,\phi_2).
$$
for every normal densities $\phi_1,\phi_2$ 
of the form
$$
\phi_i=N(p_i,1) \mbox{ for }p=(p_1,p_2) \in B(\bar p,r) \subset \R^2
$$
(and their rescaling by arbitrary
constant).
	
\textbf{STEP 2.}
We define $\phi_a(x)=\exp(-\frac{1}{2}x^2+ax)$ and $M_{ij}(a,b)=M_{ij}(\phi_a,\phi_b)$.
Thus the above implies that
\begin{equation} \label{eq:linear}
\begin{array}{r}
M_{00}(a,b)M_{11}(a,b)=M_{10}(a,b) M_{01}(a,b) 
\mbox{ for }(a,b) \in B(\bar p,r).
\end{array}
\end{equation}
	
Since
	\begin{align*}
	\frac{\partial}{\partial a}M_{i,j}(a,b)&=M_{i+1,j}(a,b),\\
	\frac{\partial}{\partial b}M_{i,j}(a,b)&=M_{i,j+1}(a,b)
	\end{align*}
by differentiating \eqref{eq:linear} with respect to the first variable we get
$$	M_{10}M_{11}+M_{00}M_{21}=M_{20}M_{01}+M_{10}M_{11},
$$
which trivially yields
$
M_{00}M_{21}=M_{20}M_{01}.
$
Analogous formula holds for the second variable, yielding
$
M_{00}M_{12}=M_{10}M_{02}.
$ 
By applying induction over indexes $i$ and $j$ we can verify that
$
M_{00}M_{ij}=M_{i0}M_{0j}.
$
By notation $m=M_{00}$, $m^1_i=M_{i0}$, $m^2_j=M_{0j}$ (moments with respect to 
only one variable), we obtain that
\begin{equation} \label{eq:formula}
M_{ij}=\frac{m_i^1}{m} \cdot  \frac{m_j^2}{m}.
\end{equation}
We apply the above for $\phi_1,\phi_2$ at $\bar p=(\bar p_1,\bar p_2)$.
	
\textbf{STEP 3.}
We consider the density of weighed dataset $\X_w$ by:
$
f_w(x_1,x_2)=\frac{1}{m}\phi_1(x_1)\phi_2(x_2) f(x_1,x_2).
$
Then the marginal densities are given by
\begin{align*}
f_1(x_1)&=\frac{1}{m}\int \phi_1(x_1) \phi_2(v_2)f(x_1,v_2)dv_2=\int f_w(x_1,v_2)dv_2, \\
f_2(x_2)&=\frac{1}{m}\int \phi_1(v_1) \phi_2(x_2)f(v_1,x_2)dv_1 =\int f_w(v_1,x_2)dv_1.
\end{align*}



\begin{figure}[!h]
\normalsize
\begin{center}
\subfigure[Original images.] {\label{fig:image_ICA_int_1}
\includegraphics[width=0.23\linewidth]{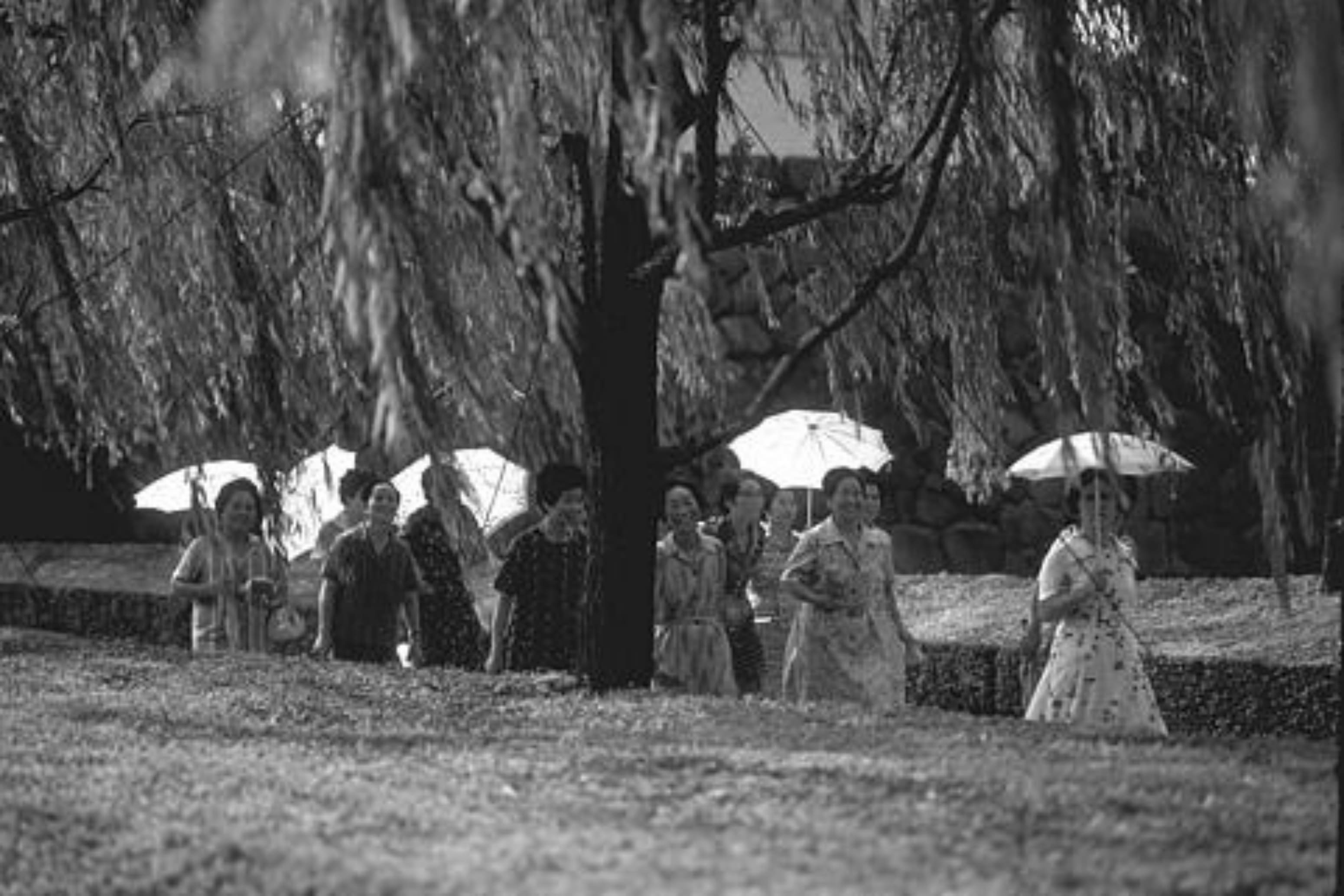}
\includegraphics[width=0.23\linewidth]{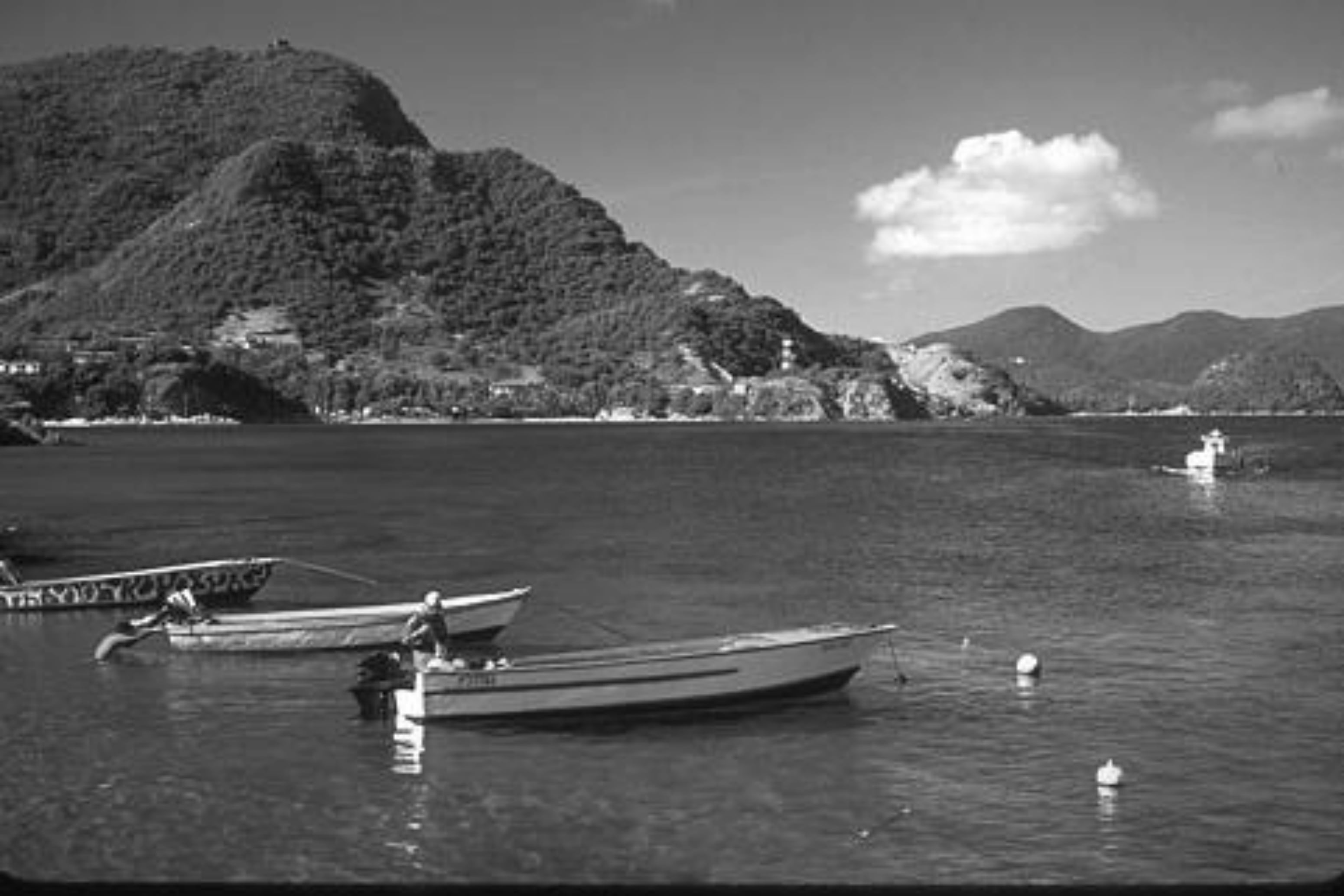}
}
\subfigure[Mix of images done via random linear projection.] {\label{fig:image_ICA_int_2}
\includegraphics[width=0.23\linewidth]{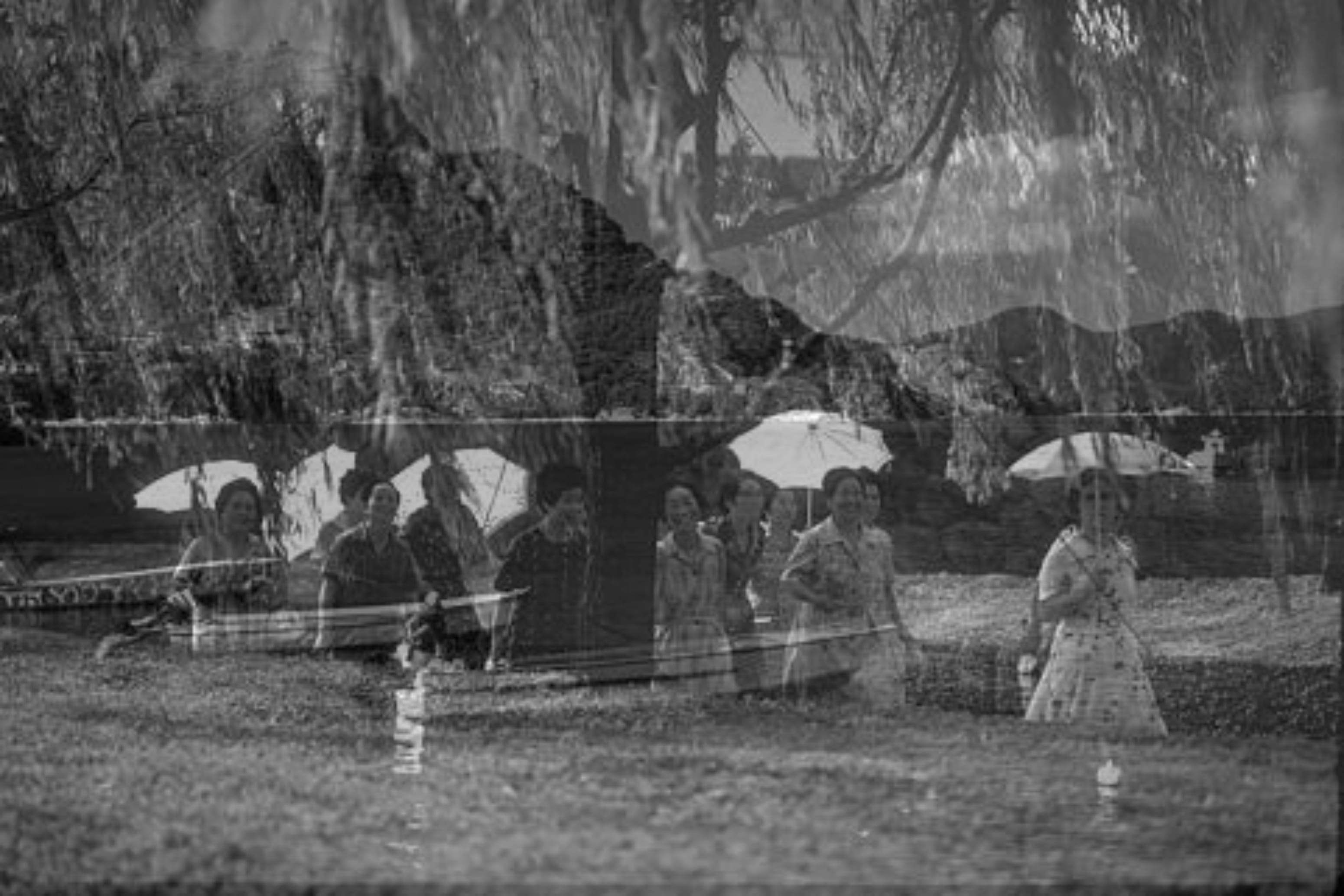}
\includegraphics[width=0.23\linewidth]{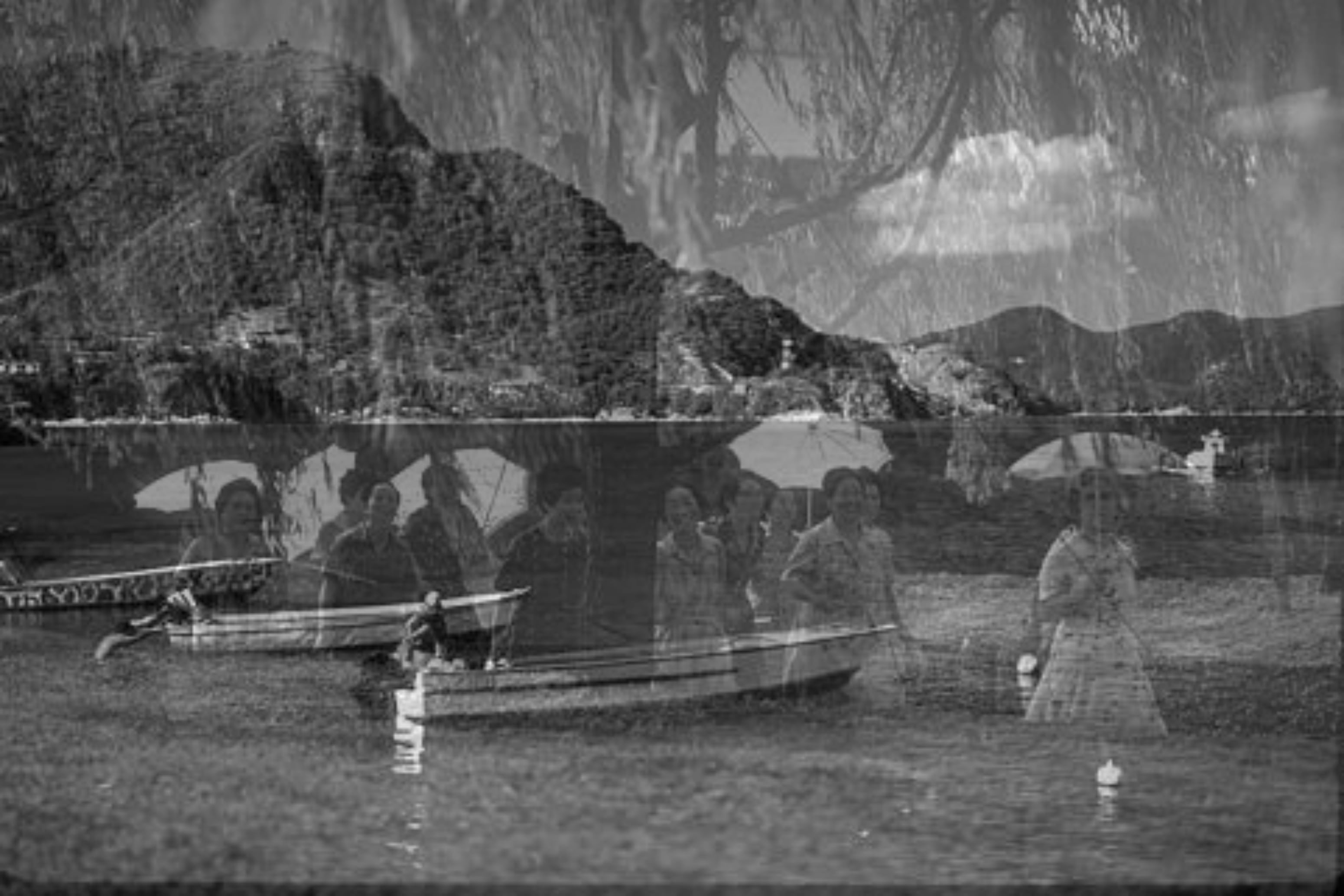}
}
\subfigure[\mWeICA.] {\label{fig:image_ICA_int_3}
\includegraphics[width=0.23\linewidth]{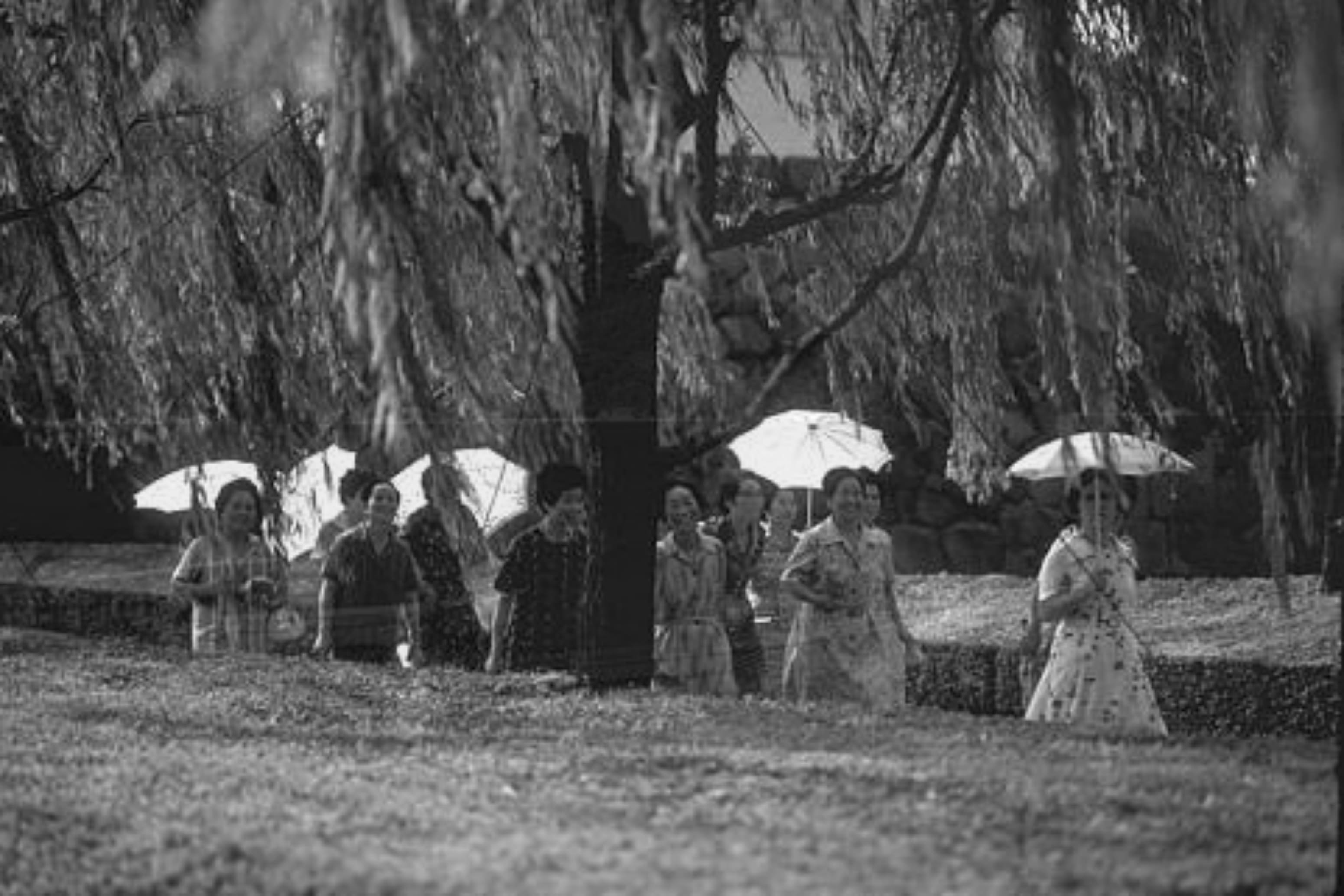} 
\includegraphics[width=0.23\linewidth]{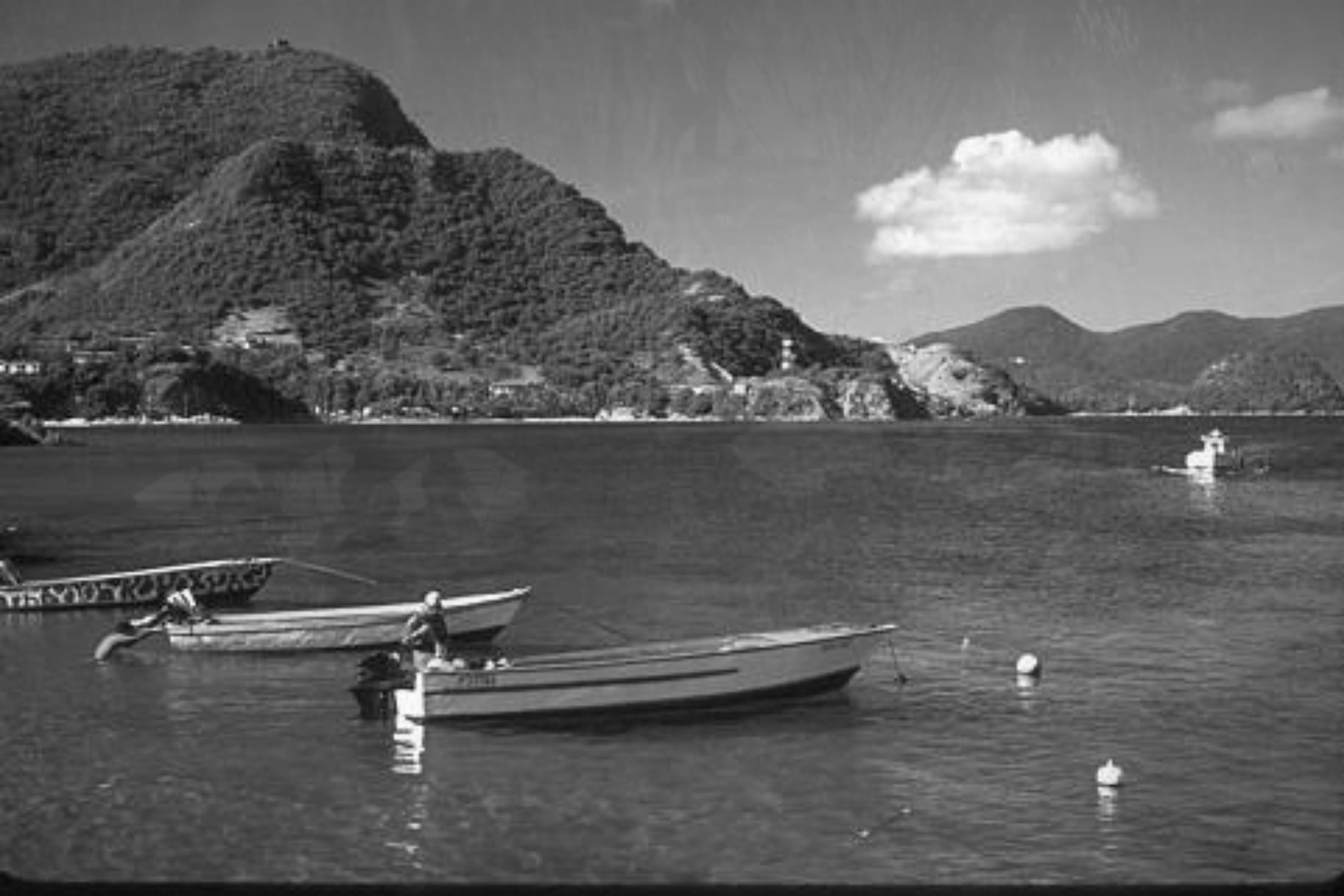} 
}
\subfigure[FastICA.] {\label{fig:image_ICA_int_4}
\includegraphics[width=0.23\linewidth]{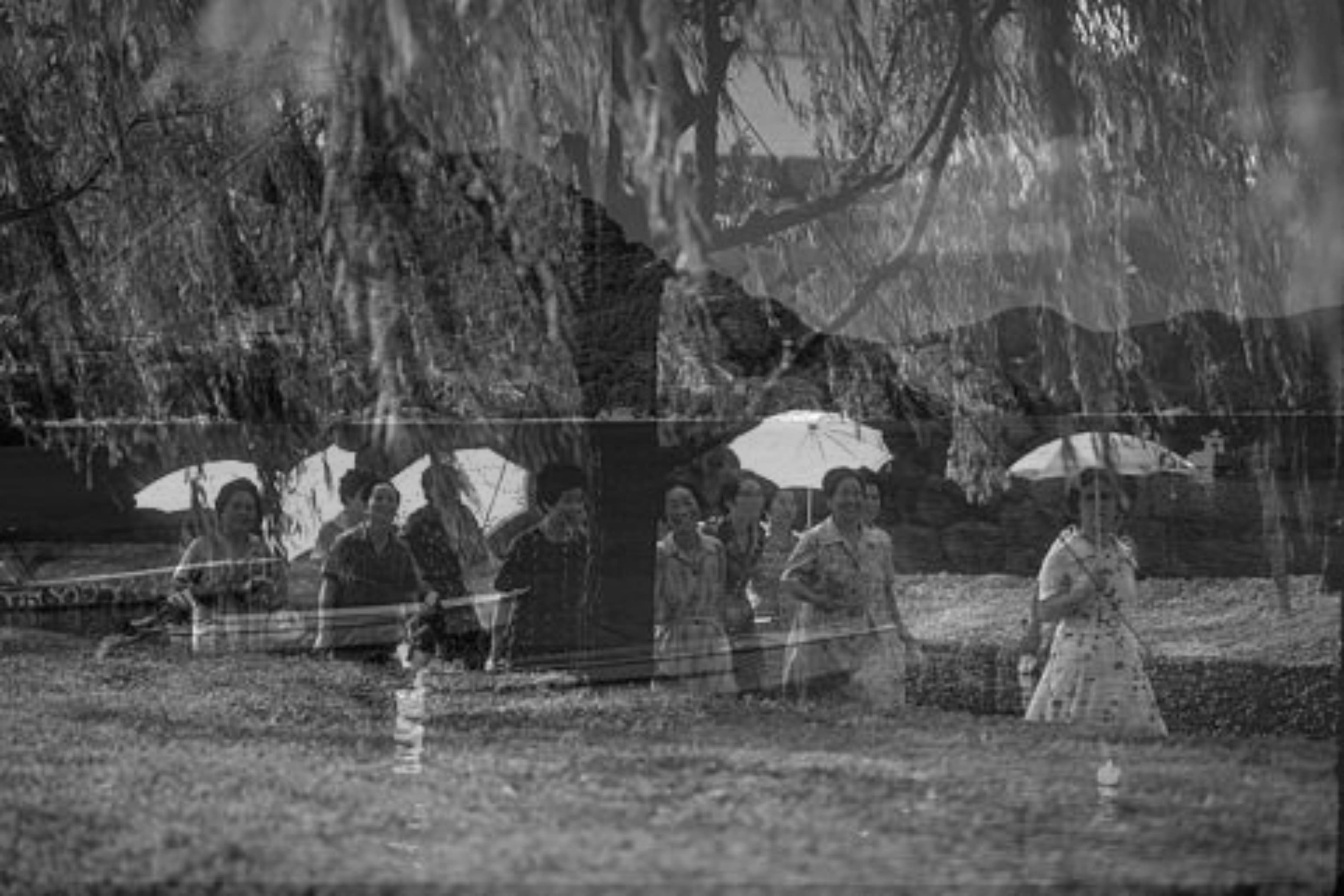} 
\includegraphics[width=0.23\linewidth]{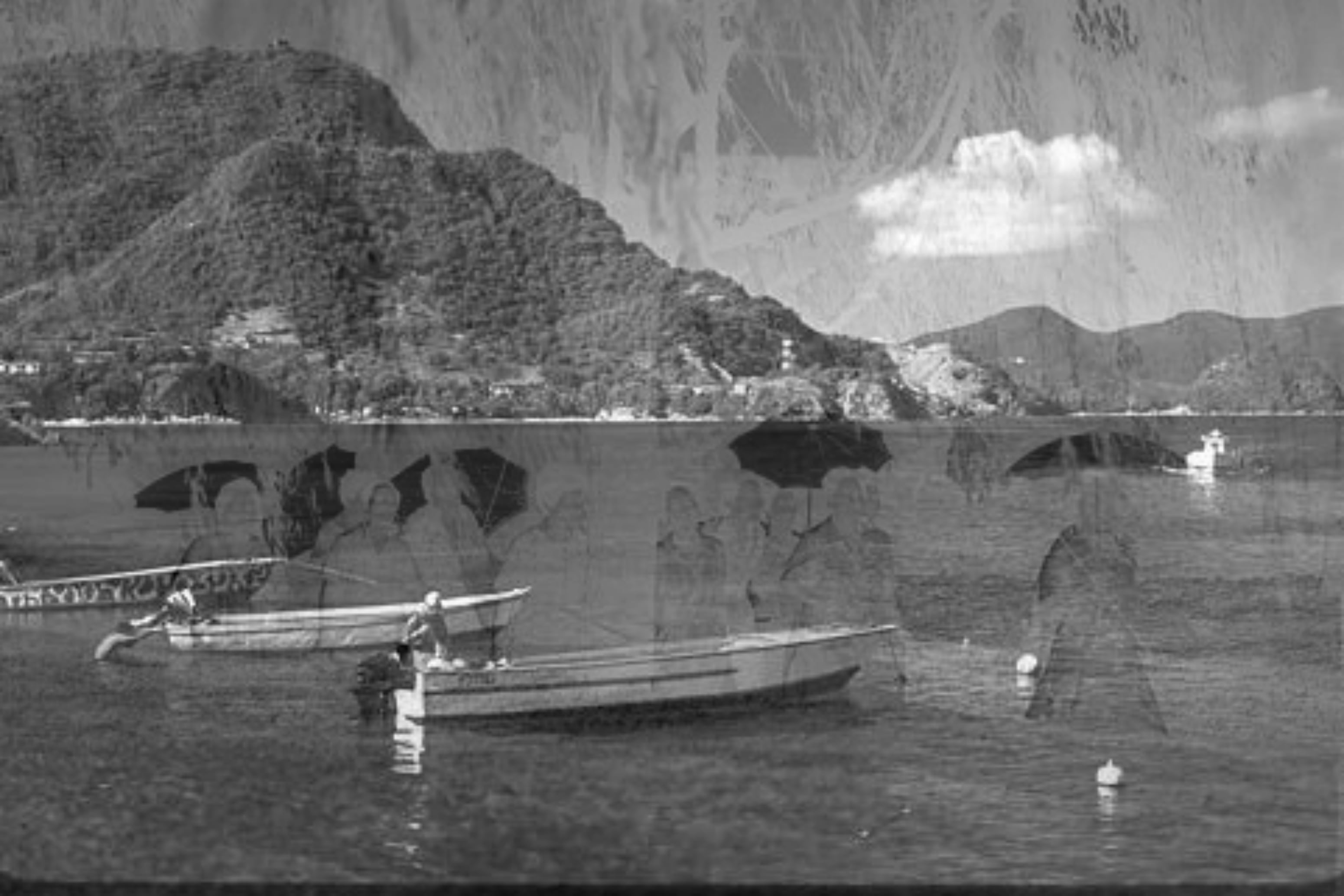} 
}
\end{center}
\caption{Comparison of images separation by \mWeICA with FastICA. One can spot that iterative approach of FastICA has some problems even with linear mixing, which was easier to solve for other two methods. }
\label{fig:worse}
\end{figure}

Let $g(x_1,x_2)=f_1(x_1)\cdot f_2(x_2)$.
Now by \eqref{eq:formula} we obtain moments of $f_w$ coincide with that of $g$:
\begin{align*}
&\iint v_1^iv_2^j f_w(v_1,v_2)dv_1dv_2 \\ 
&=\iint v_1^i f_w(v_1,v_2)dv_1dv_2 \cdot  \iint v_2^j f_w(v_1,v_2)dv_1dv_2
\\	&=\int v_1^i  \int f_w(v_1,v_2)dv_2  dv_1 \cdot \int v_2^j  \int f_w(v_1,v_2)dv_1  dv_2
\\
&=\int v_1^i f_1(v_1) dv_1 \cdot \int v_2^j f_2(v_2) dv_2=\iint v_1^i v_2^j g(v_1,v_2) dv_1 dv_2.
\end{align*}

But densities which have the same moments obviously coincide, which yields
$$
f_w(x_1,x_2)=g(x_1,x_2)=f_1(x_1) \cdot f_2(x_1).
$$
Consequently $f_w$ has independent coordinates, and therefore
$$
f(x_1,x_2)=m \cdot \frac{f_1(x_1)}{\phi_1(x_1)} \cdot
\frac{f_2(x_2)}{\phi_2(x_2)}, 
$$
which trivially yields that also $f$ has independent coordinates.
\end{proof}

\begin{remark}
Making use of the above theorem we can define a new index of independence, namely for random variable $\X$ we can define
$$
\DE(\X)=\mathbf{E} \{\DE(\cov \X_{[m]}): m \sim \X\}.
$$
Given a sample $X$ from the random variable $\X$, we can estimate the above index by computing
$$
\frac{1}{n}\sum_{i=1}^n \DE(\cov X_{[m_i]} ),
$$
where $(m_i)_{i=1..n}$ are randomly taken $n$ elements from the set $X$. 
\end{remark}

Now we proceed to the theorem which show the inverse result for Observation \ref{pr:1} also holds. Applying the previous theorem for $\Y=W^T \X$ we directly obtain the following corollary

\begin{corollary}
\label{tw:odwrotne}
We consider random vector $\X$.
We assume that an invertible square matrix $W$ is such that 
\begin{equation} \label{eq:10}
W^T \cov \X_{[m]} W
\end{equation}
is diagonal for every $m \in \R^d$.

Then $W^T \X$ has independent components.
\end{corollary}






\section{Experiments} \label{experiments}

In this section we applied our algorithms to blind source signal separation problem. We present results for \mWeICA{} on the synthetic bootstrap set and real mixes of pictures.  We will compare quality of retrieved signals from both approaches to already known solutions using rankings on Tucker Congruency Coefficient \citep{Tucker} as well as time complexity for fastest of the approaches.

\paragraph{Image separation}

Typical test for ICA task is based on the separation of mixed images. In our experiments we have used multiple images from the Berkeley Segmentation Dataset with various resolutions. 

First we took pairs of images from above source, and use them as base signals combined by mixing matrix generated separately for each pair. Clearly we need to use signals with the same resolution to make appropriate mixing. Due to the aforementioned action we obtain pair of new images. We used them as a signal, on which we perform reconstruction to base components. The main goal was to achieve separation onto original pictures, based only on those mixed signals. Exemplary results are presented in Fig. \ref{fig:worse}. 
It can be noticed in Fig. \ref{fig:worse}, that standard FastICA algorithm achieves notably worse result than our approach.

Results on that benchmark set (see Fig. \ref{fig:imgaccuracy}) shows that \mWeICA{} works very well and obtain second best score in the ranking. Only NGPP gives better score, but such method works only in reasonable small dimension (see Fig. \ref{timeComparision}).  
The difference between methods can be see as artifacts in background, see Fig. \ref{fig:worse}.

\begin{figure}[h]
\centering
    \subfigure[Original signals from EEG]{
	\includegraphics[width=0.47\linewidth]{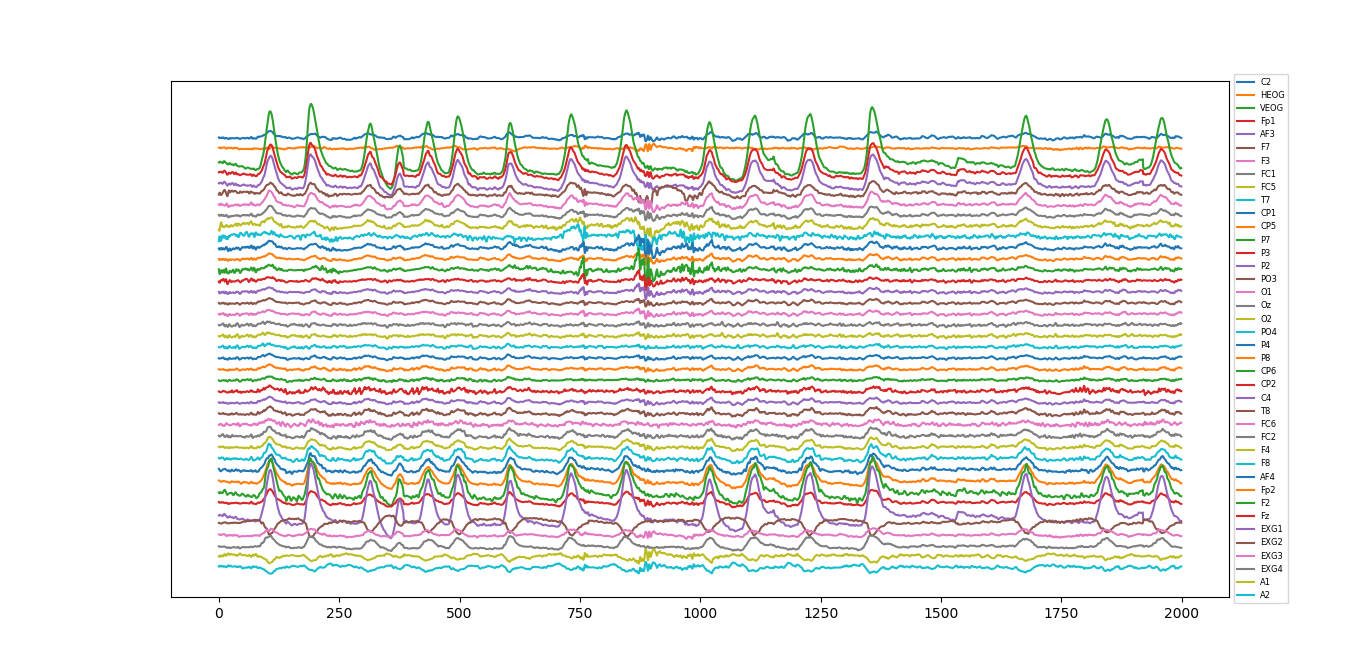}
	\label{eeg1}
	}
    \subfigure[Signals retrived from \mWeICA{}]{
	\includegraphics[width=0.47\linewidth]{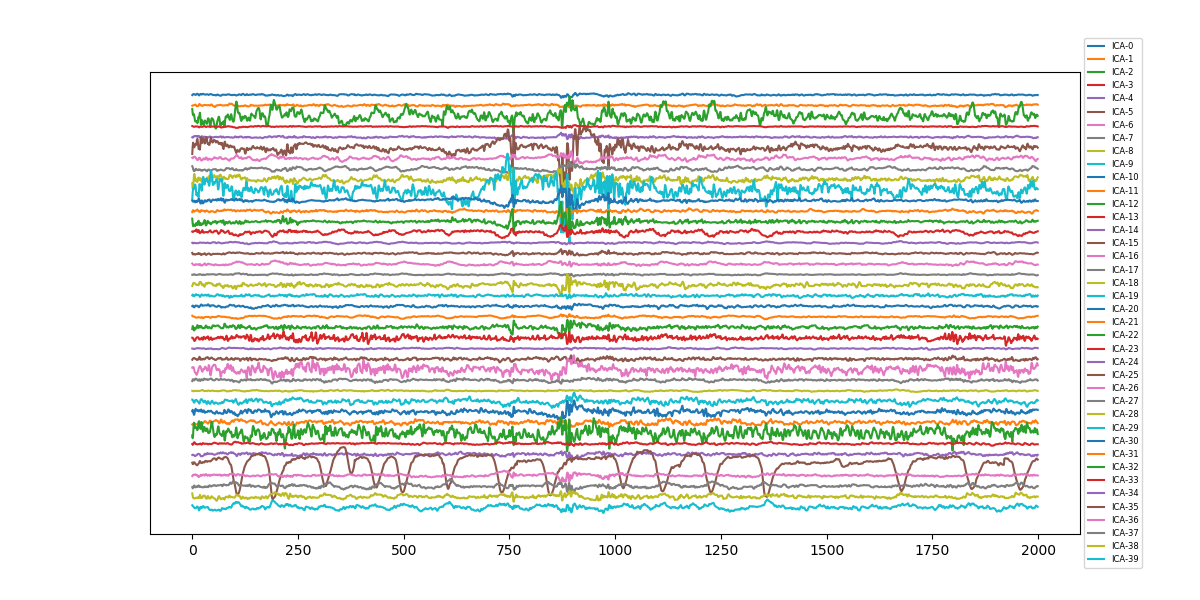}
	\label{eeg2}
	}
	\subfigure[Deleted signals from umixed EEG]{
	\includegraphics[width=0.47\linewidth]{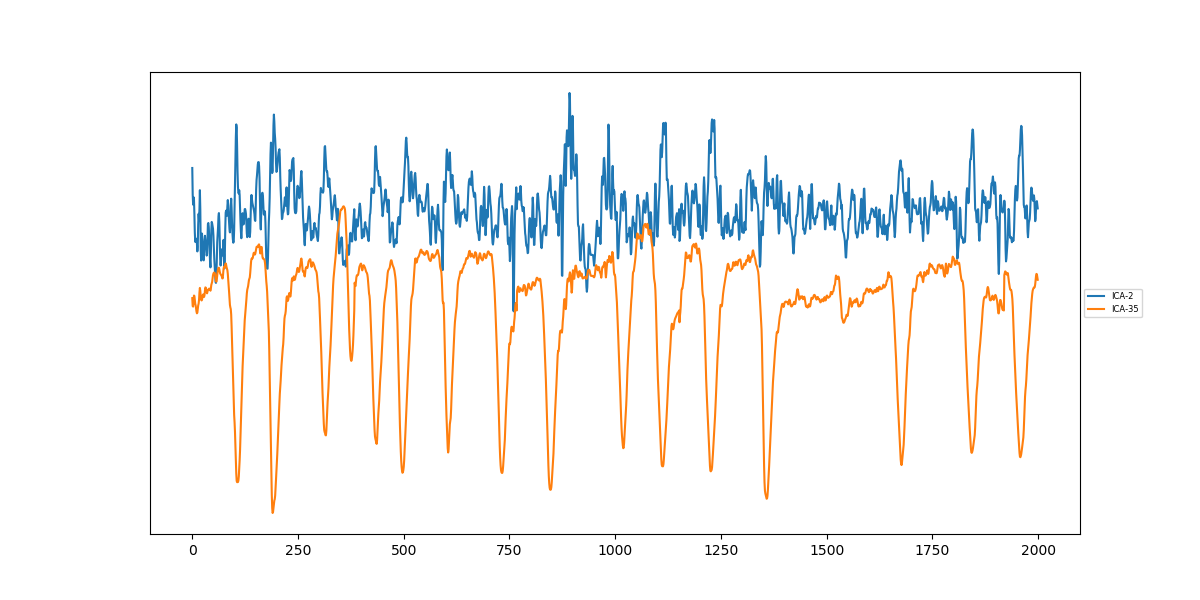}
	\label{eeg3}
	}
    \subfigure[Original EEG signal with removed components 2 and 35]{
	\includegraphics[width=0.47\linewidth]{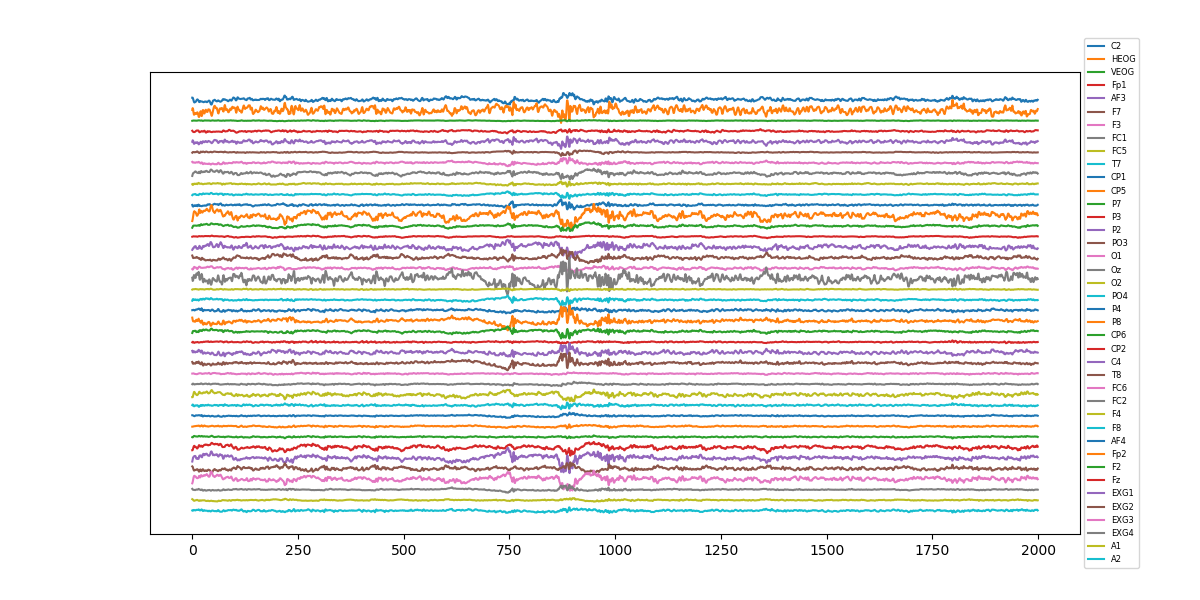}\label{eeg4}
	}
	\caption{Results of \mWeICA{} in case of EEG data.\label{eeg}}
\end{figure}

\paragraph{Computational efficiency} \label{timeCp}

We verify the computational times of WeICA and alternative ICA algorithms. We examine the influence on the number of data set instances and dimension of data. We consider the classical image separation problem, where images from the USC-SIPI Image Database (of size $512 \times 512$ pixels) are mixed together. We use
ten mixed examples and present mean evaluation times. To vary the size of data,
images are scaled to different sizes and the running times are reported in each case. 

One can observe in Fig. \ref{timeComparision} that \mWeICA{} has similar computational time as classical models with respect to dimension and one of the best one (only WeICA and JADE are more effective) in the case of number of samples.   
Summarizing, we obtained numerically effective method which gives second best score in the case of source separation problem, see Fig. \ref{fig:imgaccuracy}. 


\paragraph{Bootstrap tests}

Since image separation experiment is quite specific, we verify ICA algorithms on separating bootstrap samples task. For this purpose, we consider a real data set retrieved from UCI repository\footnote{\url{https://archive.ics.uci.edu/ml/datasets/glass+identification}} and randomly select two (and tree) coordinates to independently create 100 bootstrap samples. In the case where the distribution of the initial sample is unknown, bootstrapping is of special help in that it provides information about the distribution. Furthermore, this procedure allows to construct really independent samples. The results are again measured by Tucker's congruence coefficient.
The results presented in Fig. \ref{rank_img_2} show that \mWeICA{} obtains one of the best scores. 

\paragraph{EEG} The Electroencephalography (EEG) is an electophysiological monitoring method of recording electrical activity of the brain. In clinical contexts, EEG refers to the recording of the brain's spontaneous electrical activity over a period of time, as recorded from multiple electrodes placed on the scalp. Signals from those electrodes are mixed according to linear superposition principle. In this context ICA is used to undo the mixing \cite{EEG} and preliminary step of cleaning the data. In our experiment we focused on detection of blinking and eye movement during EEG test. 

For EEG signals, the rows of the matrix $\X$ are the signals recorded on different electrodes. Unmixed rows of the output matrix $W^T\X$  are time courses of activation of the ICA components The columns of the inverse matrix $(W^T)^{-1}$, give the projection strengths of the respective components onto the scalp sensors.

Data set of EEG signals used in our analysis was collected from 40 scalp electrodes and is presented on Fig. \ref{eeg}\protect\subref{eeg1}. Data set was analyzed in \mWeICA{} framework, and produced unmixed signals presented on Fig. \ref{eeg}\protect\subref{eeg2}. Fig. \ref{eeg}\protect\subref{eeg3} presents separated signals, which we choose as an eye blinking artifacts. After removing those two signals and going back to original sitation, one can easily spot that eye blinking spikes disappeared (Fig. \ref{eeg}\protect\subref{eeg4}) - which was our goal.

\paragraph{Sound separation} Another experiment that was performed during testing of \mWeICA{} was sound separation. We took 200 groups of signals. Each group consisted 10 signals from \href{http://marsyas.info/downloads/datasets.html}{Marsyas Music Speech data-set}. For every group distinct mixing matrix was produced and applied to produce 10 mixes of signals, which were an input for ICA methods. We expected to retrieve as much base signals as it was possible. Source sounds lasted 30 seconds, giving 10 dimensional time series containing 661500 point to analyze. As it was shown in Section \ref{timeCp} and Fig. \ref{timeComparision}, \mWeICA{} outperforms other methods in computational efficiency.

Due to high dimension of our mixtures only couple tested algorithms were capable to work in reasonable amount of time. Results presented in Fig. \ref{fig:sound-10dim} shows that \mWeICA{} retrieved comparable amount of information as the best methods.

\begin{figure}[!h]
\centering
\includegraphics[width=0.8\linewidth]{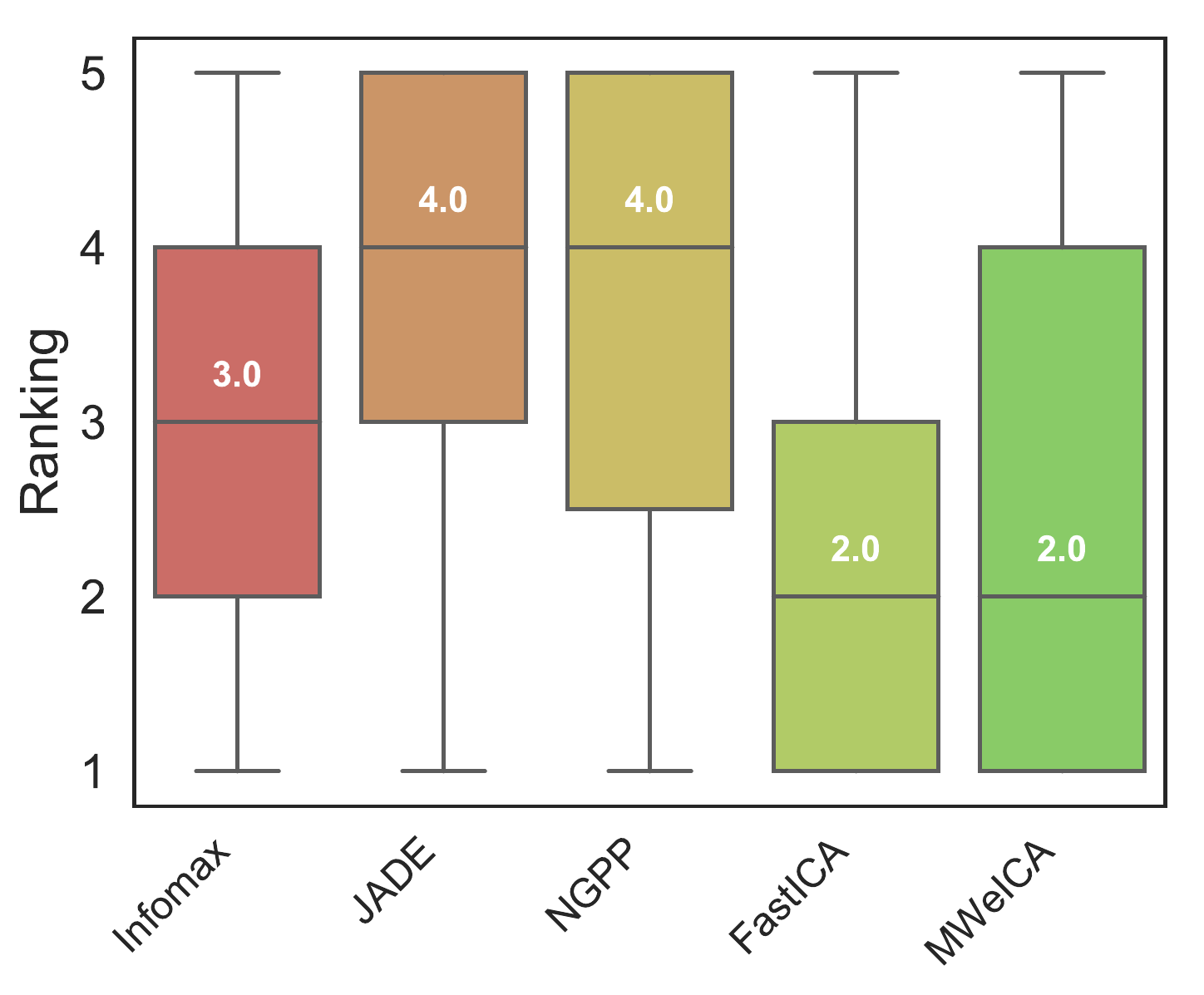} 
\caption{Results of ranking (lower is better) for Tucker Congruence Measure on sound mixes of 10 sources. Only couple fastest and best performing (in lower dimensions) methods from other examples were tested.}
	\label{fig:sound-10dim}
\end{figure}


\section{Conclusion}

In this paper we have presented \mWeICA{}, a fast ICA algorithm, which in its structure is similar to PCA. Our experiments show that \mWeICA{} achieves comparable results to state-of-the-art solutions for ICA task.

Our idea is based on theoretical result, which says that exact diagonalization of weighted covariances guarantees independence. Such result allows us to construct independence measure, which can be used in ICA framework. 
In the further work we plan to verified a possibility to use the method as a measure of independence in deep neural networks.

\bibliography{bib}
\bibliographystyle{authordate1}
\end{document}